\documentclass{article}

\usepackage{microtype}
\usepackage{graphicx}
\usepackage{subfigure}
\usepackage{booktabs} 
\usepackage{amsmath}
\usepackage{amsfonts}
\usepackage{enumitem}
\usepackage{amsthm}
\usepackage{bbm}
\usepackage{wrapfig}
\usepackage{multirow}

\newtheorem{theorem}{Theorem}

\usepackage{todonotes}

\newcommand{\method}{POEx~}
\newcommand{\methodnoindent}{POEx}

\usepackage{hyperref}

\usepackage[accepted]{icml2021}

\icmltitlerunning{Partially Observed Exchangeable Modeling}

\begin{document}

\twocolumn[
\icmltitle{Partially Observed Exchangeable Modeling}



\icmlsetsymbol{equal}{*}

\begin{icmlauthorlist}
\icmlauthor{Yang Li}{unc}
\icmlauthor{Junier B. Oliva}{unc}
\end{icmlauthorlist}

\icmlaffiliation{unc}{Department of Computer Science, University of North Carolina at Chapel Hill, NC, USA}

\icmlcorrespondingauthor{Yang Li}{yangli95@cs.unc.edu}

\icmlkeywords{Machine Learning, ICML}

\vskip 0.3in
]



\printAffiliationsAndNotice{}  

\begin{abstract}
Modeling dependencies among features is fundamental for many machine learning tasks. Although there are often multiple related instances that may be leveraged to inform conditional dependencies, typical approaches only model conditional dependencies over individual instances. In this work, we propose a novel framework, partially observed exchangeable modeling (POEx) that takes in a set of related partially observed instances and infers the conditional distribution for the unobserved dimensions over multiple elements. Our approach jointly models the intra-instance (among features in a point) and inter-instance (among multiple points in a set) dependencies in data. POEx is a general framework that encompasses many existing tasks such as point cloud expansion and
few-shot generation, as well as new tasks like few-shot imputation. Despite its generality, extensive empirical evaluations show that our model achieves state-of-the-art performance across a range of applications.
\end{abstract}

\section{Introduction}
Modeling dependencies among features is at the core of many unsupervised learning tasks. Typical approaches consider modeling dependencies in a vacuum. For example, one typically imputes the unobserved features of a single instance based only on that instance's observed features. However, there are often multiple related instances that may be leveraged to inform conditional dependencies. For instance, a patient may have multiple visits to a clinic with different sets of measurements, which may be used together to infer the missing ones. In this work, we propose to jointly model the intra-instance (among features in a point) and inter-instance (among multiple points in a set) dependencies by modeling sets of partially observed instances. To our knowledge, this is the first work that generalizes the concept of partially observed data to exchangeable sets. 
We consider modeling an exchangeable (permutation invariant) likelihood over a set $\mathbf{x}=\{x_i\}_{i=1}^{N}, x_i \in \mathbb{R}^d$. However, unlike previous approaches   \cite{korshunova2018bruno, edwards2016towards, li2020exchangeable, bender2020exchangeable}, we model a partially observed set, where ``unobserved'' features of points $\mathbf{x}_u = \{x_i^{(u_i)}\}_{i=1}^{N}$ are conditioned on ``observed'' features of points $\mathbf{x}_o = \{x_i^{(o_i)}\}_{i=1}^{N}$ and $o_i, u_i \subseteq \{1,\ldots,d\}$ and $o_i \cap u_i = \emptyset$. Since each feature in $\mathbf{x}_u$ depends not only on features from the corresponding element but also on features from other set elements, the conditional likelihood $p(\mathbf{x}_u \mid \mathbf{x}_o)$ captures the dependencies across both features and set elements.

Probabilistic modeling of sets where each instance itself contains a collection of elements is challenging, since set elements are exchangeable and the cardinality may vary. Our partially observed setting brings another level of challenge due to the arbitrariness of the observed subset for each element. First, the subsets have arbitrary dimensionality, which poses challenges for modern neural network based models. Second, the combinatorial nature of the subsets renders the conditional distributions highly multi-modal, which makes it difficult to model accurately. 
To resolve these difficulties, we propose a variational weight-sharing scheme that is able to model the combinatorial cases in a single model.

Partially observed exchangeable modeling (\methodnoindent) is a general framework that encompasses many impactful applications, which we describe below. Despite its generality, we find that our single \method approach provides competitive or better results than specialized approaches for these tasks.

\textbf{Few-shot Imputation}
A direct application of the conditional distribution $p(\mathbf{x}_u \mid \mathbf{x}_o)$ enables a task we coin \emph{few-shot imputation}, where one models a subset of covariates based on multiple related observations of an instance $\{x_i^{(o_i)}\}_{i=1}^{N}$. Our set imputation formulation leverages the dependencies across set elements to infer the missing values. For example, when modeling an occluded region in an image $x_i^{(u_i)}$, it would be beneficial to also condition on observed pixels from other angles $x_j^{(o_j)}$. 
This task is akin to multi-task imputation and is related to group mean imputation \cite{sim2015missing}, which imputes missing features in an instance according to the mean value of the features in a related group. 
However, our approach models an entire distribution (rather than providing a single imputation) and captures richer dependencies beyond the mean of the features.
Given diverse sets during training, our \method model generalizes to unseen sets.

\textbf{Set Expansion}
When some set elements have fully observed features, $o_i = \{1, \ldots, d\}$, and others have fully unobserved features, $o_j = \emptyset$, \method can generate novel elements based on the given set of illustrations. Representative examples of this application include point cloud completion and upsampling, where new points are generated from the underlying geometry to either complete an occluded point cloud or improve the resolution of a sparse point cloud.

\textbf{Few-shot Generation}
The set expansion formulation can be viewed as a few-shot generative model, where novel instances are generated based on a few exemplars. Given diverse training examples, the model is expected to generate novel instances even on unseen sets.

\textbf{Set Compression}
Instead of expanding a set, we may proceed in the opposite direction and compress a given set. For example, we can represent a large point set with its coreset to reduce storage and computing requirements. The likelihood from our \method model can guide the selection of an optimal subset, which retains the most information.

\textbf{Neural Processes}
If we introduce an index variable $t_i$ for each set element and extend the original set $\{x_i\}_{i=1}^{N}$ to a set of index-value pairs $\{(t_i, x_i)\}_{i=1}^{N}$, our \method model encapsulates the neural processes \cite{garnelo2018conditional,garnelo2018neural,kim2019attentive} as a special case. New elements corresponding to the given indexes can be generated from a conditional version of \methodnoindent: $p(\mathbf{x}_u \mid \mathbf{x}_o,\mathbf{t})$, where $\mathbf{t}=\{t_i\}_{i=1}^{N}$. In this work, we focus on modeling processes in high-dimensional spaces, such as processes of images, which are challenging due to the multi-modality of the underlying distributions.

\textbf{Set of Functions}
Instead of modeling a set of finite dimensional vectors, we may be interested in sets of functions, such as a set of correlated processes. By leveraging the dependencies across functions, we can fit each function better while utilizing fewer observations. Our formulation essentially generalizes the multi-task Gaussian processes \cite{bonilla2008multi} into multi-task neural processes.

The contributions of this work are as follows: 1) We extend the concept of partially observed data to exchangeable sets so that the dependencies among both features and set elements are captured in a single model. 2) We develop a deep latent variable based model to learn the conditional distributions for sets and propose a collapsed inference technique to optimize the ELBO. The collapsed inference simplifies the hierarchical inference framework to a single level. 3) We leverage the captured dependencies to perform various applications, which are difficult or even impossible for alternative approaches. 4) Our model handles neural processes as special cases and generalizes the original neural processes to high-dimensional distributions. 5) We propose a novel extension of neural process, dubbed multi-task neural process, where sets of infinite-dimensional functions are modeled together. 6) We conduct extensive experiments to verify the effectiveness of our proposed model and demonstrate state-of-the-art performance across a range of applications.

\section{Background}

\textbf{Set Modeling}
The main challenge of modeling set structured data is to respect the permutation invariant property of sets. A straight-forward approach is to augment the training data with randomly permuted orders and treat them as sequences. Given infinite training data and model capacity, an autoregressive model can produce permutation invariant likelihoods. However, for real-world limited data and models, permutation invariance is not guaranteed. As pointed out in \cite{vinyals2015order}, the order actually matters for autoregressive models.

BRUNO \cite{korshunova2018bruno} proposes using invertible transformations to project each set element to a latent space where dimensions are factorized independently. Then they build independent exchangeable processes for each dimension in the latent space to obtain the permutation invariant likelihoods. FlowScan \cite{bender2020exchangeable} instead recommends using a scan sorting operation to convert the set likelihood to a familiar sequence likelihood and normalizing the likelihood accordingly. ExNODE \cite{li2020exchangeable} utilizes neural ODE based permutation equivariant transformations and permutation invariant base likelihoods to construct a continuous normalizing flow model for exchangeable data.

De Finetti's theorem provides a principled way of modeling exchangeable data, where each element is modeled independently conditioned on a latent variable $\theta$:
\begin{equation}\label{eq:de-finetti}
\resizebox{0.7\linewidth}{!}{$
    p(\{x_i\}_{i=1}^{N}) = \int \prod_{i=1}^{N} p(x_i \mid \theta)p(\theta)d\theta.$}
\end{equation}
Latent Dirichlet allocation (LDA) \cite{blei2003latent} and its variants \cite{teh2006hierarchical,blei2007correlated} are classic models of this form, where the likelihood and prior are expressed as simple known distributions. Recently, deep neural network based models have been proposed \cite{yang2019pointflow,edwards2016towards,yang2020energy}, in which a VAE is trained to optimize a lower bound of \eqref{eq:de-finetti}. 

\textbf{Arbitrary Conditional Models}
Instead of modeling the joint distribution $p(x)$, where $x \in \mathbb{R}^d$, arbitrary conditional models learn the conditional distributions for an arbitrary subset of features $x_u$ conditioned on another non-overlapping arbitrary subset $x_o$, where $u,o \subseteq \{1,\ldots,d\}$. Graphical models are a natural choice for such tasks \cite{koster2002marginalizing}, where conditioning usually has a closed-form solution. However, the graph structure is usually unknown for general data, and learning the graph structure from observational data has its own challenges \cite{heinze2018causal,scutari2019learns}. Sum-Product Network (SPN) \cite{poon2011sum} and its variants \cite{jaini2018deep,butz2019deep,tan2019hierarchical} borrow the idea from graphical models and build deep neural networks by stacking sum and product operations alternately so that the arbitrary conditionals are tractable.

Deep generative models have also been used for this task. Universal Marginalizer \cite{douglas2017universal} builds a feed-forward network to approximate the conditional marginal distributions of each dimension conditioned on $x_o$. VAEAC \cite{ivanov2018variational} utilizes a conditional VAE to learn the conditional distribution $p(x_u \mid x_o)$. 
ACFlow \cite{li2019flow} uses a normalizing flow based model for learning the arbitrary conditionals, where invertible transformations are specially designed to deal with arbitrary dimensionalities. GAN based approaches \cite{belghazi2019learning} have also been proposed to model arbitrary conditionals. 

\textbf{Stochastic Processes}
Stochastic processes are usually defined as the marginal distribution over a collection of indexed random variables $\{x_t; t\in \mathcal{T}\}$. For example, Gaussian process \cite{rasmussen2003gaussian} specifies that the marginal distribution $p(x_{t_1:t_n} \mid \{t_i\}_{i=1}^{n})$ follows a multivariate Gaussian distribution, where the covariance is defined by some kernel function $K(t, t')$. The Kolmogrov extension theorem \cite{oksendal2003stochastic} provides the sufficient condition for designing a valid stochastic process:
\begin{itemize}[leftmargin=*,noitemsep,topsep=0pt]
\item \textbf{Exchangeability:} The marginal distribution is invariant to any permutation $\pi$, i.e.,
\begin{equation*}
    p(x_{t_1:t_n} \mid \{t_i\}_{i=1}^{n}) = p(x_{t_{\pi_1}:t_{\pi_n}} \mid \pi(\{t_i\}_{i=1}^{n})).
\end{equation*}
\item \textbf{Consistency:} Marginalizing out part of the variables is the same as the one obtained from the original process, i.e., for any $1 \leq m \leq n$ 
\begin{equation*}
\resizebox{\linewidth}{!}{$
    p(x_{t_1:t_m} \mid \{t_i\}_{i=1}^{m}) = \int p(x_{t_1:t_n} \mid \{t_i\}_{i=1}^{n}) dx_{t_{m+1}:t_n}.$}
\end{equation*}
\end{itemize}
Stochastic processes can be viewed as a distribution over the space of functions and can be used for modeling exchangeable data. However, classic Gaussian processes \cite{rasmussen2003gaussian} and Student-t processes \cite{shah2014student} assume the marginals follow a simple known distribution for tractability and have an $O(n^3)$ complexity, which render them impractical for large-scale complex dataset.

Neural Processes \cite{garnelo2018conditional,garnelo2018neural,kim2019attentive} overcome the above limitations by learning a latent variable based model conditioned on a set of context points $X^{(C)} = \{(t_i^{(C)}, x_i^{(C)})\}_{i=1}^{N_C}$. The latent variable $\theta$ implicitly parametrizes a distribution over the underlying functions so that values on target points $\{t_j^{(T)}\}_{j=1}^{N_T}$ can be evaluated over random draws of the latent variable, i.e.,
\begin{equation*}
\resizebox{\linewidth}{!}{
    $p(x_{t_1:t_{N_T}}^{(T)} \mid \{t_j^{(T)}\}_{j=1}^{N_T}) = \int \prod_{j=1}^{N_T} p(x_j^{(T)} \mid t_j^{(T)}, \theta) p(\theta \mid X^{(C)}) d\theta.$}
\end{equation*}
Neural processes generalize the kernel based stochastic processes with deep neural networks and scale with $O(n)$ due to the amortized inference. The exchangeability requirement is met by using exchangeable neural networks for inference, and the consistency requirement is roughly satisfied with the variational approximation.

\section{Method}
In this section, we develop our approach for modeling sets of partially observed elements. We describe the variants of \method and their corresponding applications. We also introduce our inference techniques used to train the model.

\subsection{Partially Observed Exchangeable Modeling}

\subsubsection{Arbitrary Conditionals}
Consider a set of vectors $\{x_i\}_{i=1}^{N}$, where $x_i \in \mathbb{R}^d$ and $N$ is the cardinality of the set. For each set element $x_i$, only a subset of features $x_i^{(o_i)}$ are observed and we would like to predict the values for another subset of features $x_i^{(u_i)}$. Here, $u_i,o_i \subseteq \{1, \ldots,d\}$ and $u_i \cap o_i = \emptyset$. We denote the set of observed features as $\mathbf{x}_o = \{x_i^{(o_i)}\}_{i=1}^{N}$ and the set of unobserved features as $\mathbf{x}_u = \{x_i^{(u_i)}\}_{i=1}^{N}$. Our goal is to model the distribution $p(\mathbf{x}_u \mid \mathbf{x}_o)$ for arbitrary $u_i$ and $o_i$.

In order to model the arbitrary conditional distributions for sets, we introduce a latent variable $\theta$. The following theorem states that there exists a latent variable $\theta$ such that conditioning on $\theta$ renders the set elements of $\mathbf{x}_u$ i.i.d.. Please see appendix for the proof.
\begin{theorem}\label{thm:ac_thm}
Given a set of observations $\mathbf{x}=\{x_i\}_{i=1}^{N}$ from an infinitely exchangeable process, denote the observed and unobserved part as $\mathbf{x}_o = \{x_i^{(o_i)}\}_{i=1}^{N}$ and $\mathbf{x}_u = \{x_i^{(u_i)}\}_{i=1}^{N}$ respectively. Then the arbitrary conditional distribution $p(\mathbf{x}_u \mid \mathbf{x}_o)$ can be decomposed as follows:
\begin{equation}\label{eq:ac_set}
\resizebox{0.75\linewidth}{!}{$\displaystyle{
    p(\mathbf{x}_u \mid \mathbf{x}_o) = \int \prod_{i=1}^{N} p(x_i^{(u_i)} \mid x_i^{(o_i)}, \theta) p(\theta \mid \mathbf{x}_o) d\theta.}$}
\end{equation}
\end{theorem}
Optimizing \eqref{eq:ac_set}, however, is intractable due to the high-dimensional integration over $\theta$. Therefore, we resort to variational approximation and optimize a lower bound:
\begin{equation}\label{eq:ac_set_elbo}
\resizebox{0.8\linewidth}{!}{$
\begin{aligned}
    \log p(\mathbf{x}_u \mid \mathbf{x}_o) \geq& \sum_{i=1}^{N} \mathbb{E}_{q(\theta \mid \mathbf{x}_u, \mathbf{x}_o)} \log p(x_i^{(u_i)} \mid x_i^{(o_i)}, \theta) \\&- D_{KL}(q(\theta \mid \mathbf{x}_u, \mathbf{x}_o)\ \| \  p(\theta \mid \mathbf{x}_o)),
\end{aligned}$}
\end{equation}
where $q(\theta \mid \mathbf{x}_u, \mathbf{x}_o)$ and $p(\theta \mid \mathbf{x}_o)$ are variational posterior and prior that are permutation invariant w.r.t. the conditioning set. The arbitrary conditional likelihoods $p(x_i^{(u_i)} \mid x_i^{(o_i)}, \theta)$ are over a $\mathbb{R}^d$ feature space and can be implemented as in previous works \cite{ivanov2018variational,li2019flow,belghazi2019learning}. 

Note that $\mathbf{x}_o$ and $\mathbf{x}_u$ are sets of vectors with arbitrary dimensionality. To represent vectors with arbitrary dimensionality so that a neural network can handle them easily, we impute missing features with zeros and introduce a binary mask to indicate whether the corresponding dimensions are missing or not. We denote the zero imputation operation as $\mathcal{I}(\cdot)$ that takes in a set of features with arbitrary dimensionality and outputs a set of $d$-dimensional features and the corresponding set of binary masks.

\subsubsection{Set Compression}
Given a pretrained \method model, we can use the arbitrary conditional likelihoods $p(\mathbf{x}_u \mid \mathbf{x}_o)$ to guide the selection of a subset for compression. The principle is to select a subset that preserves the most information. Set compression is a combinatorial optimization problem, which is NP-hard. Here, we propose a sequential approach that selects one element at a time. We start from $o_i=\emptyset, u_i=\{1,\ldots,d\}$ for each element. The next element $i$ to select should be the one that maximizes the conditional mutual information $I(x_i, \mathbf{x}_u \mid \mathbf{x}_o)$ so that it provides the most information for the remaining elements. From the definition of conditional mutual information, we have
\begin{equation*}
\resizebox{\linewidth}{!}{$
    I(x_i, \mathbf{x}_u \mid \mathbf{x}_o) = H(x_i \mid \mathbf{x}_o) - \mathbb{E}_{p(\mathbf{x}_u \mid \mathbf{x}_o)} H(x_i \mid \mathbf{x}) = H(x_i \mid \mathbf{x}_o).$}
\end{equation*}
Since the original set $\mathbf{x}$ is given, we can estimate the entropy with
\begin{equation*}
\resizebox{\linewidth}{!}{$
    H(x_i \mid \mathbf{x}_o)= \mathbb{E}_{p(x_i \mid \mathbf{x}_o)} -\log p(x_i \mid \mathbf{x}_o) \approx -\log p(x_i \mid \mathbf{x}_o).$}
\end{equation*}
Therefore, the next element to chose is simply the one with minimum likelihood $p(x_i \mid \mathbf{x}_o)$ based on the current selection $\mathbf{x}_o$. Afterwards, we update $o_i=\{1,\ldots,d\}, u_i=\emptyset$ and proceed to the next selection step.

\subsubsection{Neural Process}
Some applications may introduce index variables for each set element. For example, a collection of frames from a video are naturally indexed by their timestamps. Here, we consider a set of index-value pairs $\{(t_i, x_i)\}_{i=1}^{N}$, where $t_i$ can be either discrete or continuous. Similarly, $x_i$ are partially observed, and we define $\mathbf{x}_u$ and $\mathbf{x}_o$ accordingly. We also define $\mathbf{t} = \{t_i\}_{i=1}^{N}$ for notation simplicity, which are typically given. By conditioning on the index variables $\mathbf{t}$, we modify the lower bound \eqref{eq:ac_set_elbo} to
\begin{equation}\label{eq:np_elbo}
\resizebox{0.9\linewidth}{!}{$
\begin{aligned}
    \log p(\mathbf{x}_u \mid \mathbf{x}_o, \mathbf{t}) \geq& \sum_{i=1}^{N} \mathbb{E}_{q(\theta \mid \mathbf{x}_u, \mathbf{x}_o, \mathbf{t})} \log p(x_i^{(u_i)} \mid x_i^{(o_i)}, t_i, \theta) \\&- D_{KL}(q(\theta \mid \mathbf{x}_u, \mathbf{x}_o, \mathbf{t}) \| p(\theta \mid \mathbf{x}_o, \mathbf{t})).
\end{aligned}$}
\end{equation}

If we further generalize the cardinality $N$ to be infinite and specify a context set $\mathbf{x}_c$ and a target set $\mathbf{x}_t$ to be arbitrary subsets of all set elements, i.e., $\mathbf{x}_c, \mathbf{x}_t \subseteq \{(t_i, x_i)\}_{i=1}^{N}$, we recover the exact setting for neural process. This is a special case of our \method model in that features are fully observed ($o_i=\{1,\ldots,d\},u_i=\emptyset$) for elements of $\mathbf{x}_c$ and fully unobserved ($o_i=\emptyset,u_i=\{1,\ldots,d\}$) for elements of $\mathbf{x}_t$. That is, $\mathbf{x}_o = \mathbf{x}_c$ and $\mathbf{x}_u = \mathbf{x}_t$. The ELBO objective is exactly the same as \eqref{eq:np_elbo}. Similar to neural processes, we use a finite set of data points to optimize the ELBO \eqref{eq:np_elbo} and sample a subset at random as the context.

Neural processes usually use simple feed-forward networks and Gaussian distributions for the conditional likelihood $p(x_i^{(u_i)} \mid x_i^{(o_i)}, t_i, \theta)$, which makes it unsuitable for multi-modal distributions. Furthermore, they typically deal with low-dimensional data. Our model, however, utilizes arbitrary conditional likelihoods, which can deal with high-dimensional and multi-modal distributions.

\subsubsection{Multi-task Neural Process}
Neural processes model the distributions over functions, where one input variable is mapped to one target variable. In a multi-task learning scenario, there exists multiple target variables. Therefore, we propose a multi-task neural process extension to capture the correlations among target variables. For notation simplicity, we assume the target variables are exchangeable here. Non-exchangeable targets can be easily transformed to exchangeable ones by concatenating with their indexes. Consider a set of functions $\{\mathcal{F}_k\}_{k=1}^{K}$ for $K$ target variables. Inspired by neural process, we represent each function $\mathcal{F}_k$ by a set of input-output pairs $\{(t_{ki},x_{ki})\}_{i=1}^{N_k}$. The goal of multi-task neural process is to learn an arbitrary conditional model given arbitrarily observed subsets from each function. We similarly define $\mathbf{x}_u = \{\mathcal{F}_k^{(u_k)}\}_{k=1}^{N_k} = \{\{(t_{ki}, x_{ki}^{(u_{ki})})\}_{i=1}^{N_k}\}_{k=1}^{K}$ and $\mathbf{x}_o = \{\mathcal{F}_k^{(o_k)}\}_{k=1}^{N_k} = \{\{(t_{ki}, x_{ki}^{(o_{ki})})\}_{i=1}^{N_k}\}_{k=1}^{K}$.

The multi-task neural process described above models a set of sets. A straight-forward approach is to use a hierarchical model
\begin{equation}\label{eq:mtnp_hier}
\resizebox{0.9\linewidth}{!}{$
\begin{aligned}
    &p(\mathbf{x}_u \mid \mathbf{x}_o) = \int \prod_{k=1}^{K} p(\mathcal{F}_k^{(u_k)} \mid \mathcal{F}_k^{(o_k)}, \theta) p(\theta \mid \mathbf{x}_o) d\theta = \\ &\int \prod_{k=1}^{K} \left[ \int \prod_{i=1}^{N_k} p(x_{ki}^{(u_{ki})} \mid x_{ki}^{(o_{ki})}, t_{ki}, \phi) p(\phi \mid \mathcal{F}_k^{(o_k)}, \theta) d\phi \right] p(\theta \mid \mathbf{x}_o) d\theta,
\end{aligned}$}
\end{equation}
which utilizes the Theorem \ref{thm:ac_thm} twice. However, inference with such a model is challenging since complex inter-dependencies need to be captured across two set levels. Moreover, the latent variables are not of direct interest. Therefore, we propose an inference technique that collapses the two latent variables into one. Specifically, we assume the uncertainties across $\theta$ and $\phi$ are both absorbed into $\theta$ and define $p(\phi \mid \mathcal{F}_k^{(o_k)}, \theta) = \delta(G(\mathcal{F}_k^{(o_k)}, \theta))$, where $G$ represents a deterministic mapping. Therefore, \eqref{eq:mtnp_hier} can be simplified as 
\begin{equation}
\resizebox{0.9\linewidth}{!}{$
\begin{aligned}
    &p(\mathbf{x}_u \mid \mathbf{x}_o) =\\ &\int \prod_{k=1}^{K} \prod_{i=1}^{N_k} p(x_{ki}^{(u_{ki})} \mid x_{ki}^{(o_{ki})}, t_{ki}, \phi) \delta(G(\mathcal{F}_k^{(o_k)}, \theta)) p(\theta \mid \mathbf{x}_o) d\phi d\theta.
\end{aligned}$}
\end{equation}
Further collapsing $\phi$ and $\theta$ into one latent variable $\psi$ gives
\begin{equation}\label{eq:mtnp_collapsed}
\resizebox{0.9\linewidth}{!}{$\displaystyle{
    p(\mathbf{x}_u \mid \mathbf{x}_o) = \int \prod_{k=1}^{K} \prod_{i=1}^{N_k} p(x_{ki}^{(u_{ki})} \mid x_{ki}^{(o_{ki})}, t_{ki}, \psi) p(\psi \mid \mathbf{x}_o) d\psi}$,}
\end{equation}
where $\psi$ is permutation invariant for both set levels. The collapsed model may seem restricted at first sight, but we show empirically that it remains powerful when we use a flexible likelihood model for the arbitrary conditionals. More importantly, it significantly simplifies the implementation.

A similar collapsed inference technique has been used in \cite{griffiths2004finding,teh2006collapsed,porteous2008fast} to reduce computational cost and accelerate inference for LDA models. Recently, \citet{yang2020energy} propose to use collapsed inference in the neural process framework to marginalize out the index variables. Here, we utilize collapsed inference to reduce a hierarchical generative model to a single level.

Given the generative process \eqref{eq:mtnp_collapsed}, it is straightforward to optimize using the ELBO
\begin{equation}\label{eq:mtnp_elbo}
\resizebox{0.9\linewidth}{!}{$
\begin{aligned}
    \log p(\mathbf{x}_u \mid \mathbf{x}_o) \geq &\sum_{k=1}^{K} \sum_{i=1}^{N_k} \mathbb{E}_{q(\psi \mid \mathbf{x}_u,\mathbf{x}_o,\mathbf{t})} \log p(x_{ki}^{(u_{ki})} \mid x_{ki}^{(o_{ki})}, t_{ki}, \psi) \\
    &- D_{KL}(q(\psi \mid \mathbf{x}_u,\mathbf{x}_o,\mathbf{t}) \| p(\psi \mid \mathbf{x}_o,\mathbf{t})).
\end{aligned}$}
\vspace{-4pt}
\end{equation}

\begin{figure*}
    \centering
    \includegraphics[width=0.7\linewidth]{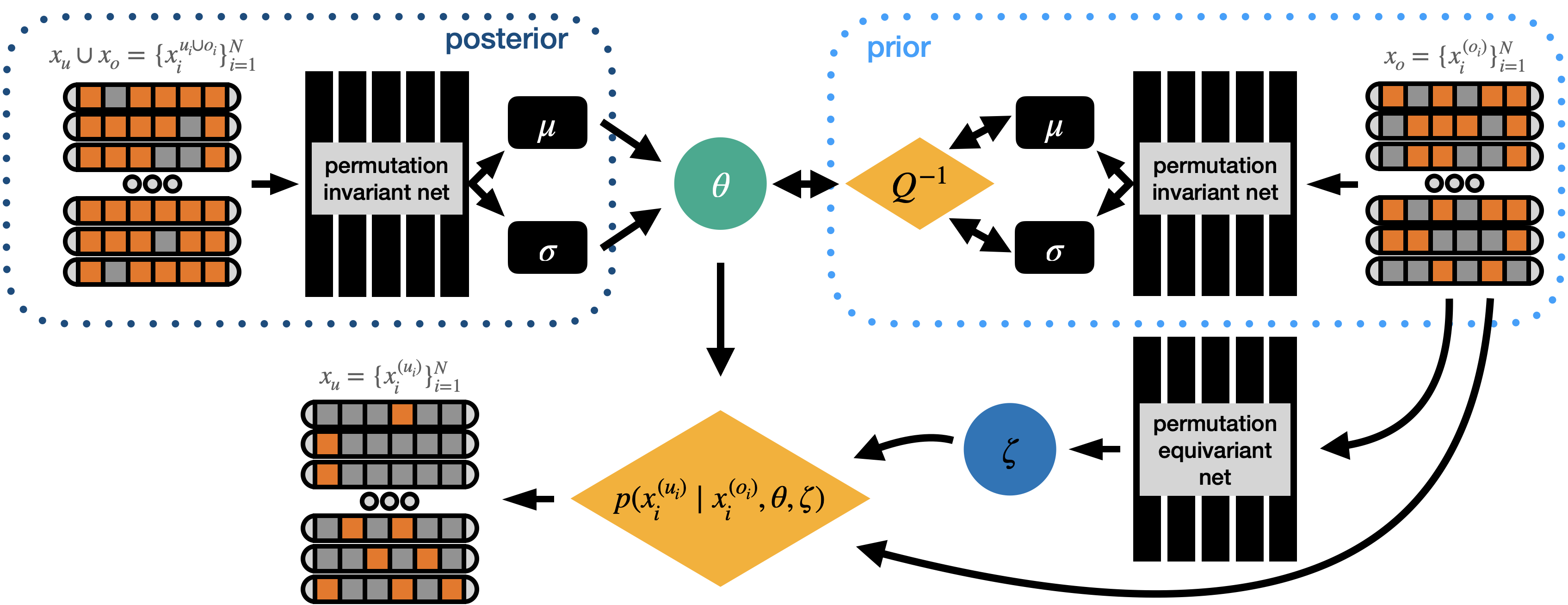}
    \caption{VAE model for partially observed exchangeable modeling.}
    \label{fig:vae}
\end{figure*}

\subsection{Implementation}
In this section, we describe some implementation details of \method that are important for good empirical performance. Please refer to Sec.~\ref{sec:nets} in the appendix for more details. Given the ELBO objectives defined in \eqref{eq:ac_set_elbo}, \eqref{eq:np_elbo} and \eqref{eq:mtnp_elbo}, it is straightforward to implement them as conditional VAEs. Please see Fig.~\ref{fig:vae} for an illustration. The posterior and prior are implemented with permutation invariant networks. For sets of vectors (such as point clouds), we first use Set Transformer \cite{lee2019set} to extract a permutation equivariant embedding, then average over the sets. For sets of images, we use a convolutional neural network to process each image independently and take the mean embedding over the sets. For sets of sets/functions, Set Transformer (with a global pooling) is used to extract the embedding for each function respectively, then the average embedding is taken as the final permutation invariant embedding. Index variables are tiled as the same sized tensor as the corresponding inputs and concatenated together. The posterior is then defined as a Gaussian distribution, where the mean and variance are derived from the set representation. The prior is defined as a normalizing flow model $Q$ with base distribution defined as a Gaussian conditioned on the set representation. The KL-divergence terms are calculated by Monte-Carlo estimation:
$D_{KL}(q \| p) = -H(q) - \mathbb{E}_{q} \log p$, where both $H(q)$ and $\log p$ are tractable.

In addition to the permutation invariant latent code, we also use a permutation equivariant embedding of $x_o$ to assist the learning of the arbitrary conditional likelihood. For a set of vectors, we use Set Transformer to capture the inter-dependencies. For images, Set Transformer is computationally too expensive. Therefore, we propose to decompose the computation across spatial dimensions and set dimension. Specifically, for a set of images $\{x_i\}_{i=1}^{N}$, shared convolutional layers are applied to each set element, and self-attention layers are applied to each spatial position. Such layers and pooling layers can be stacked alternately to extract a permutation equivariant embedding for a set of images. For sets of sets $\{\{x_{ki}\}_{i=1}^{N_k}\}_{k=1}^{K}$, the permutation equivariant embedding contains two parts. One part is the self attention embedding that attends only to the features in the same set (i.e., same $k$) $\{\textbf{SelfAttention}(\{x_{ki}\}_{i=1}^{N_k})\}_{k=1}^{K}$. Note that \textbf{SelfAttention} outputs a feature vector for each element, which is a weighted sum of a certain embedding from each element. Another part is the attention embedding across different sets $\{\frac{1}{K-1}\sum_{k'=1}^{K}\textbf{CrossAttention}(\{x_{ki}\}_{i=1}^{N_k}, \{x_{k'j}\}_{j=1}^{N_{k'}})\mathbbm{1}(k \neq k')\}_{k=1}^{K}$. For each element in the query set, the cross attention outputs an attentive embedding over the key set.

Given the permutation equivariant embedding of $\mathbf{x}_o$ (denoted as $\zeta$), the arbitrary conditionals $p(x_i^{(u_i)} \mid x_i^{(o_i)}, \theta)$, $p(x_i^{(u_i)} \mid x_i^{(o_i)}, t_i, \theta)$ and $p(x_{ki}^{(u_{ki})} \mid x_{ki}^{(o_{ki})}, t_{ki}, \psi)$ in \eqref{eq:ac_set_elbo}, \eqref{eq:np_elbo} and \eqref{eq:mtnp_elbo} are rewritten as $p(x_i^{(u_i)} \mid x_i^{(o_i)}, \theta, \zeta)$, $p(x_i^{(u_i)} \mid x_i^{(o_i)}, t_i, \theta, \zeta)$ and $p(x_{ki}^{(u_{ki})} \mid x_{ki}^{(o_{ki})}, t_{ki}, \psi, \zeta)$ respectively, which can be implemented by any arbitrary conditional models \cite{li2019flow,ivanov2018variational,douglas2017universal}. Here, we choose ACFlow for most of the experiments and modify it to a conditional version, where both the transformations and the base likelihood are conditioned on the corresponding tensors. For low-dimensional data, such as 1D function approximation, a simple feed-forward network that maps the conditioning tensor to a Gaussian distribution also works well.

\section{Experiments}
In this section, we conduct extensive experiments with \method for the aforementioned applications. In order to verify the effectiveness of the set level dependencies, we compare to a model that treats each set element as independent input (denoted as IDP). IDP uses the same architecture as the decoder of \methodnoindent. We also compare to some specially designed approaches for each application. Due to space limitations, we put the experimental details and additional results in the appendix. In this work, we emphasize the versatility of \methodnoindent, and note that certain domain specific techniques may further improve performance, which we leave for future works.

\begin{figure}
    \centering
    \subfigure[MNIST]{
    \includegraphics[width=\linewidth]{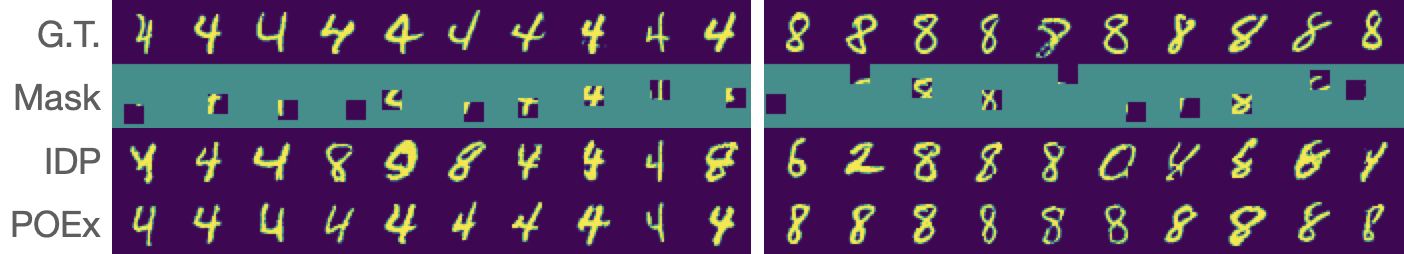}}
    \subfigure[Omniglot]{
    \includegraphics[width=\linewidth]{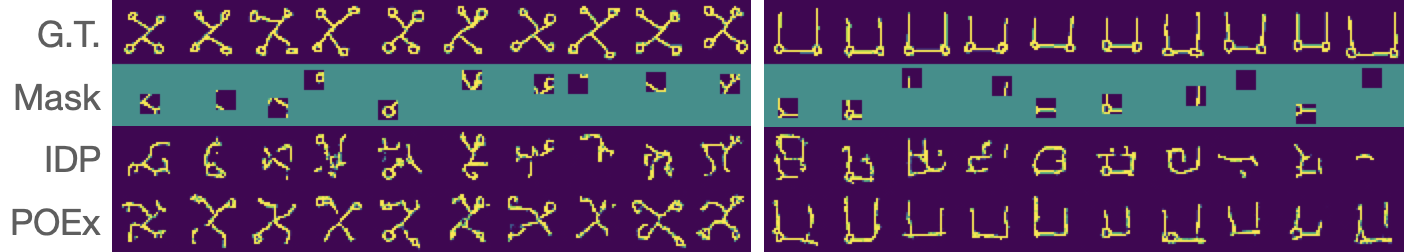}}
    \subfigure[Omniglot from unseen classes]{
    \includegraphics[width=\linewidth]{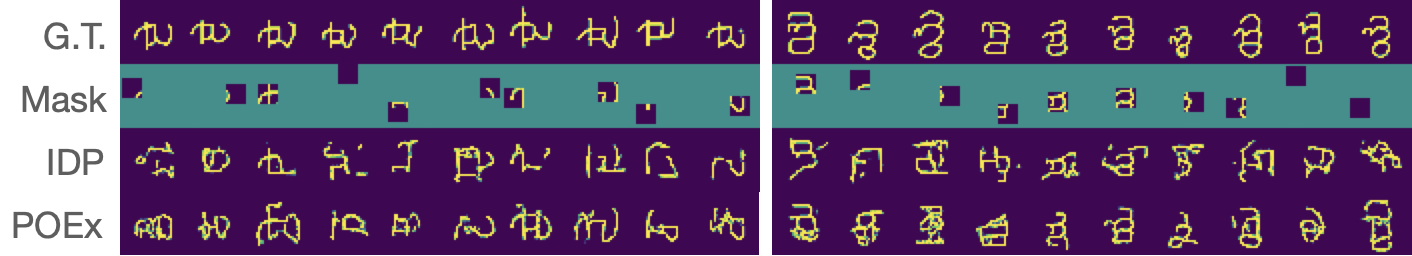}\label{fig:inpainting_unseen}}
    \vspace{-15pt}
    \caption{Inpaint the missing values for a set of images.}
    \label{fig:inpainting}
    \vspace{-15pt}
\end{figure}

\begin{wraptable}{r}{0.5\linewidth}
    \centering
    \vspace{-15pt}
    \caption{PSNR of inpainting sets of images.}
    \label{tab:inpainting}
    \small
    \begin{tabular}{c|c|c}
    \toprule
         &  MNIST & Omniglot\\
    \midrule
        TRC & 7.80 & 8.87 \\
        IDP & 11.38 & 11.49 \\
        \method & \textbf{13.02} & \textbf{12.09} \\
    \bottomrule
    \end{tabular}
\end{wraptable}

We first utilize our \method model to impute missing values for a set of images from MNIST and Omniglot datasets, where several images from the same class are considered a set. We consider a setting where only a small portion of pixels are observed for each image. Figure \ref{fig:inpainting} and Table \ref{tab:inpainting} compare the results for \methodnoindent, IDP, and a tensor completion based approach TRC \cite{wang2017efficient}. The results demonstrate clearly that the dependencies among set elements can significantly improve the imputation performance. Even when the given information is limited for each image, our \method model can still accurately recover the missing parts. TRC fails to recover any meaningful structures for both MNIST and Omniglot, see Fig.~\ref{fig:inpainting_supp} for several examples. Our \method model can also perform few-shot imputation on unseen classes, see Fig.~\ref{fig:inpainting_unseen} for several examples.

\begin{figure}[h]
    \centering
    \includegraphics[width=\linewidth]{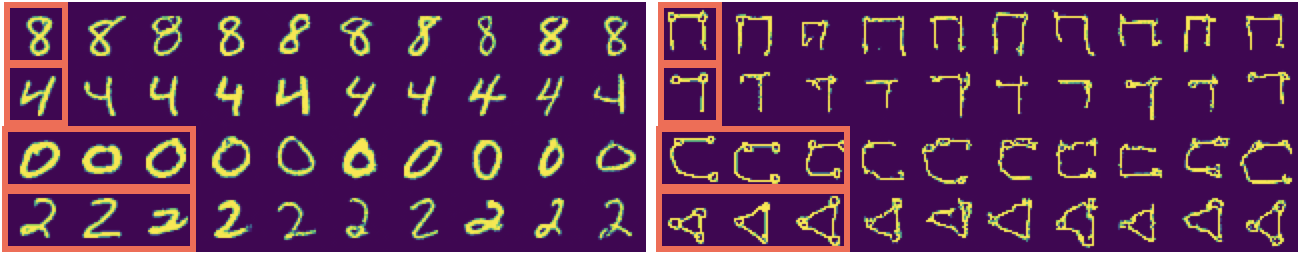}
    \vspace{-15pt}
    \caption{Expand a set by generating similar elements. Red boxes indicate the given elements. Left: MNIST. Right: Omniglot.}
    \label{fig:expansion}
\end{figure}

If we change the distribution of the masks so that some elements are fully observed, our \method model can perform set expansion by generating similar elements to the given ones. Figure \ref{fig:expansion} shows several examples for MNIST and Omniglot datasets. Our \method model can generate realistic and novel images even if only one element is given.

\begin{figure}[h]
    \centering
    \includegraphics[width=\linewidth]{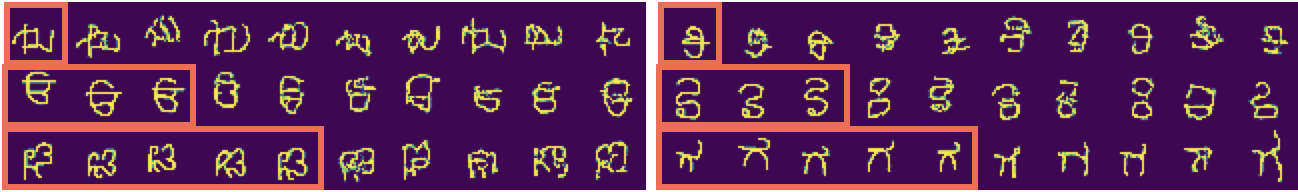}
    \vspace{-15pt}
    \caption{Few-shot generation with unseen Omniglot characters.}
    \label{fig:fewshot}
\end{figure}

\begin{wraptable}{r}{0.45\linewidth}
\vspace{-10pt}
    \centering
    \vspace{-10pt}
    \caption{5-way-1-shot classification with MAML.}
    \label{tab:fewshot}
    \small
    \begin{tabular}{cc}
    \toprule
        Algorithm &  Acc.\\
    \midrule
        MAML & 89.7 \\
        MAML(aug=5) & 93.8 \\
        MAML(aug=10) & 94.7 \\
        MAML(aug=20) & 95.1 \\
    \bottomrule
    \end{tabular}
\end{wraptable}

To further test the generalization ability of \methodnoindent, we provide the model with several unseen characters and utilize the \method model to generate new elements. Figure \ref{fig:fewshot} demonstrates the few-shot generation results given several unseen Omniglot images. We can see the generated images appear similar to the given ones. To quantitatively evaluate the quality of generated images, we perform few-shot classification by augmenting the few-shot support sets with our \method model. We evaluate the 5-way-1-shot accuracy of a fully connected network using MAML \cite{finn2017model}. Table \ref{tab:fewshot} reports the accuracy of MAML with and without augmentation. We can see the few-shot accuracy improves as we provide more synthetic data.

\begin{figure}
    \centering
    \subfigure[completion]{
    \includegraphics[width=0.45\linewidth]{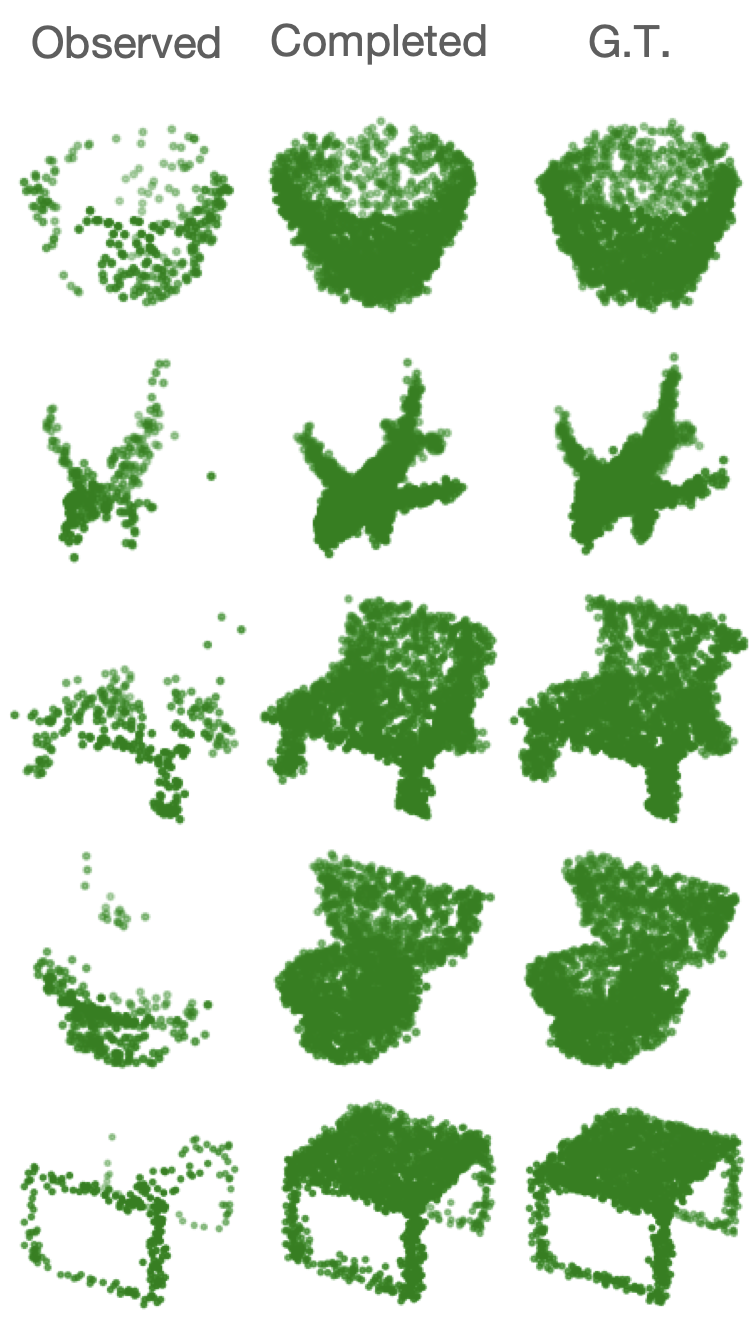}}
    \hfill
    \subfigure[upsampling]{
    \includegraphics[width=0.5\linewidth]{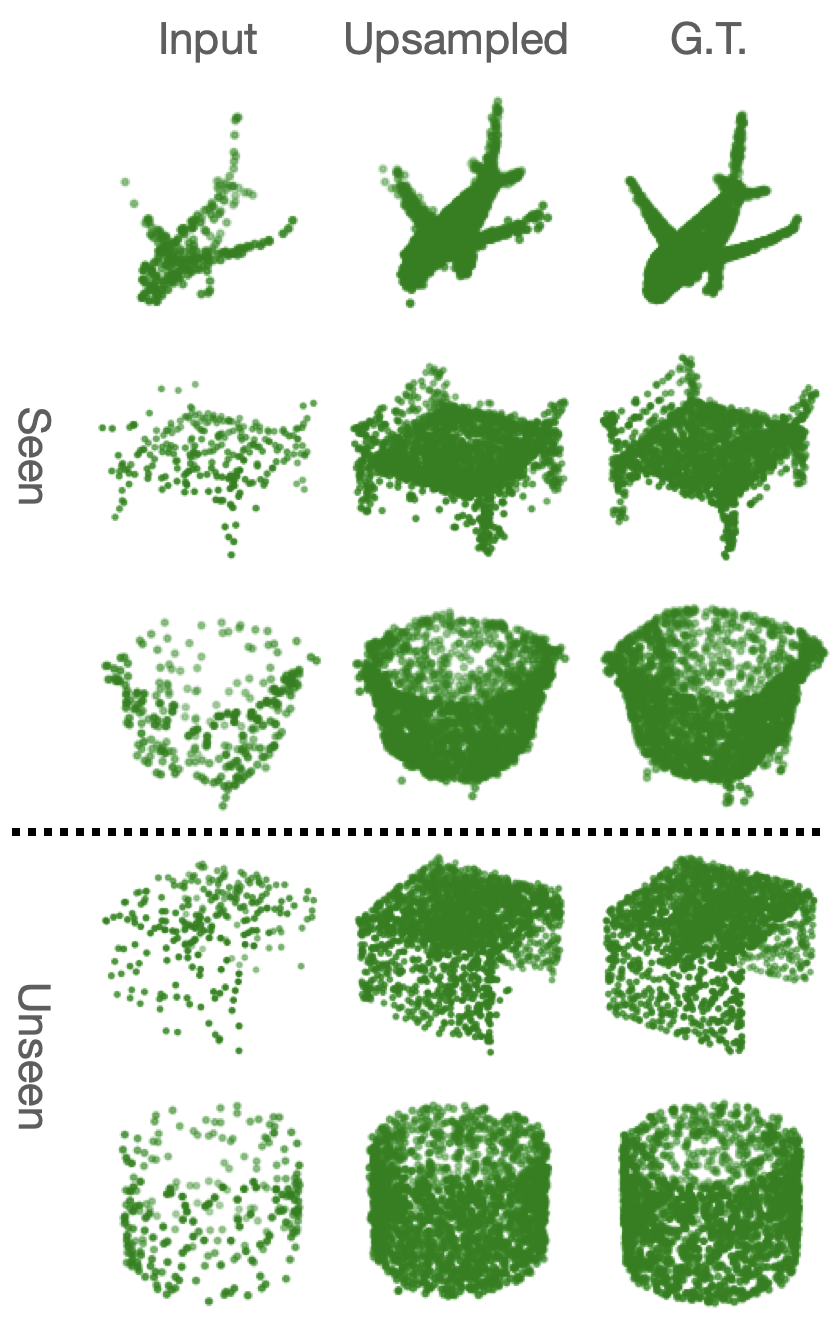}}
    \vspace{-8pt}
    \caption{Point cloud completion and upsampling.}
    \label{fig:pointcloud}
    \vspace{-15pt}
\end{figure}

In addition to sets of images, our \method model can deal with point clouds. Figure \ref{fig:pointcloud} presents several examples for point cloud completion and upsampling. Point cloud completion predicts the occluded parts based on a partial point cloud. 
\begin{wraptable}{r}{0.49\linewidth}
    \centering
    \vspace{-15pt}
    \caption{Point cloud completion.}
    \label{tab:completion}
    \footnotesize
    \begin{tabular}{ccc}
    \toprule
         &  CD & EMD\\
    \midrule
        PCN & 0.0033 & 0.1393\\
        \method & 0.0044 & 0.0994\\
    \bottomrule
    \end{tabular}
\end{wraptable}
Partial point clouds are common in practice due to limited sensor resolution and occlusion. We use the dataset created by \citet{wang2020pre}, where the point cloud is self occluded due to a single camera view point. We sample 256 points uniformly from the observed partial point cloud to generate 2048 points from the complete one using our \method model. For comparison, we train a PCN \cite{yuan2018pcn} using the same dataset. PCN is specially designed for the completion task and uses a multi-scale generation process. For quantitative comparison, we report the Chamfer Distance (CD) and Earth Mover's Distance (EMD) in Table \ref{tab:completion}. Despite the general purpose of our \method model, we achieve comparable performance compared to PCN.

\begin{wraptable}{r}{0.5\linewidth}
    \centering
    \vspace{-15pt}
    \caption{Point cloud upsampling.}
    \label{tab:upsampling}
    \tiny
    \begin{tabular}{c|c|cc}
    \toprule
        & & PUNet & \method \\
    \midrule
        \multirow{2}{*}{Seen} & CD & 0.0025 & 0.0035 \\
        & EMD & 0.0733 & 0.0880 \\
    \midrule
        \multirow{2}{*}{Unseen} & CD & 0.0031 & 0.0048 \\
        & EMD & 0.0793 & 0.1018\\
    \bottomrule
    \end{tabular}
    \vspace{-5pt}
\end{wraptable}

For point cloud upsampling, we use the ModelNet40 dataset. We uniformly sample 2048 points as the target and create a low resolution point cloud by uniformly sampling a subset. Note we use arbitrary sized subset during training. For evaluation, we upsample a point cloud with 256 points. We use PUNet \cite{yu2018pu} as the baseline, which is trained to upsample 256 points to 2048 points. Table \ref{tab:upsampling} reports the CD and EMD between the upsampled point clouds and the ground truth. We can see our \method model produces slightly higher distances, but we stress that our model is not specifically designed for this task, nor was it trained w.r.t. these metrics. We believe some task specific tricks, such as multi-scale generation and folding \cite{yang2018foldingnet}, can help improve the performance further, which we leave as future work. Similar to the image case, we can also generalize a pretrained \method model to upsample point clouds in unseen categories.

In contrast to upsampling, we propose using our \method model to compress a given point cloud. Here we use a \method model trained for airplane to summarize 2048 points into 256 points. To showcase the significance of leveraging set dependencies, we simulate a non-uniformly captured point cloud, where points close to the center have higher probability of being captured by the sensor.
\begin{wrapfigure}{r}{0.63\linewidth}
    \centering
    \vspace{-8pt}
    \includegraphics[width=\linewidth]{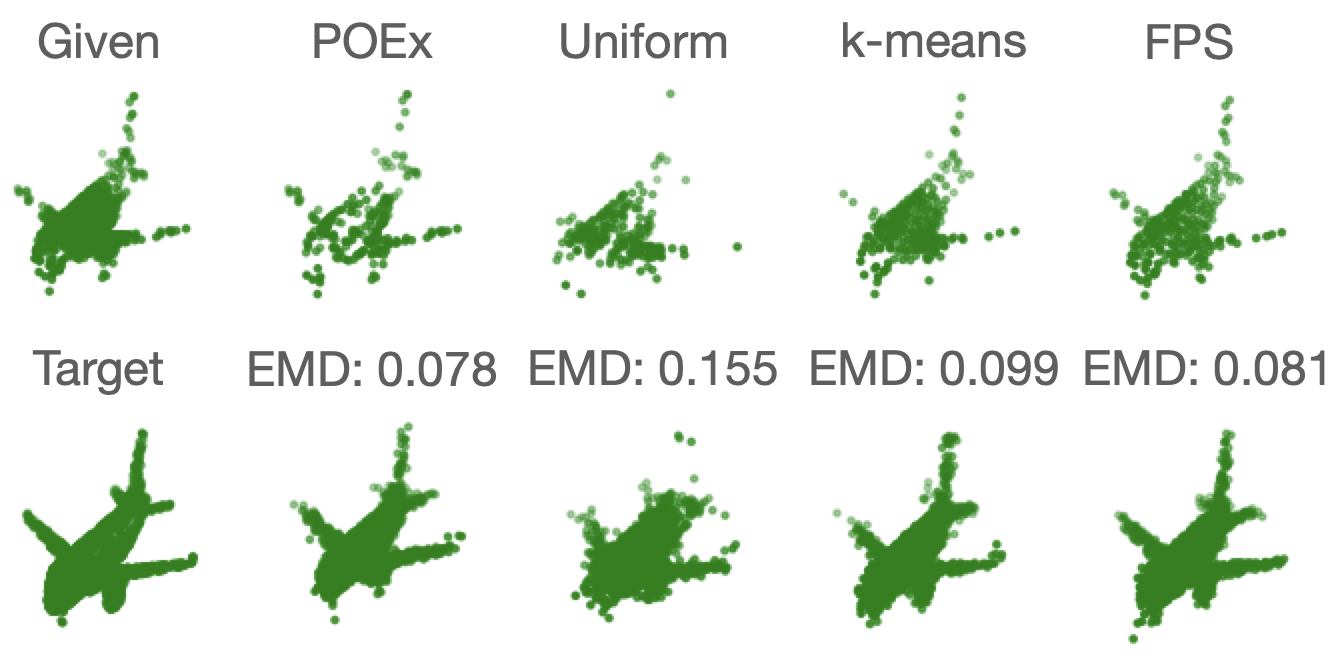}
    \vspace{-20pt}
    \caption{Point cloud compression.}
    \label{fig:compression}
    \vspace{-5pt}
\end{wrapfigure}
We expect the compressed point cloud to preserve the underlying geometry, thus we evaluate the distance between the recovered point cloud and a uniformly sampled one. Figure \ref{fig:compression} compares the compression performance with several sampling approaches, where FPS represents the farthest point sampling \cite{qi2017pointnet}. We can see the baselines tend to select center points more frequently, while \method distributes the points more evenly. Quantitative results (Fig.\ref{fig:compression}) verify the superiority of \method for compression.

\begin{wrapfigure}{r}{0.5\linewidth}
\vspace{-12pt}
\begin{minipage}{0.5\linewidth}
    \includegraphics[width=\linewidth]{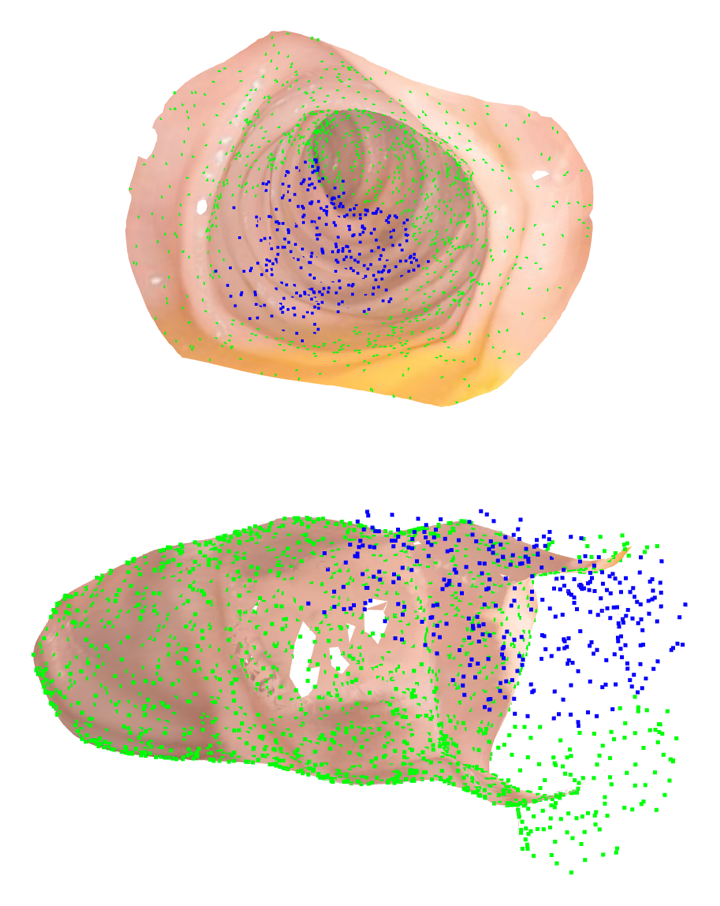}
\end{minipage}
\begin{minipage}{0.48\linewidth}
    \vspace{-10pt}
    \caption{Impute missing points for colonoscopy data. Green: observed. Blue: imputed.}
    \label{fig:colon}
\end{minipage}
\vspace{-10pt}
\end{wrapfigure}

In addition to these synthetic point cloud data, we also evaluate on a real-world colonoscopy dataset. We uniformly sample 2048 points from the meshes and manually drop some points to simulate the blind spots. Our \method model is then used to predict those missing points. To provide guidance about where to fill in those missing points, we divide the space into small cubes and pair each point with its cube coordinates. Missing points are then predicted conditioned on their cube coordinates. Figure \ref{fig:colon} presents several imputed point clouds mapped onto their corresponding meshes. We can see the imputed points align well with the meshes.

\begin{figure}
    \centering
    \includegraphics[width=\linewidth]{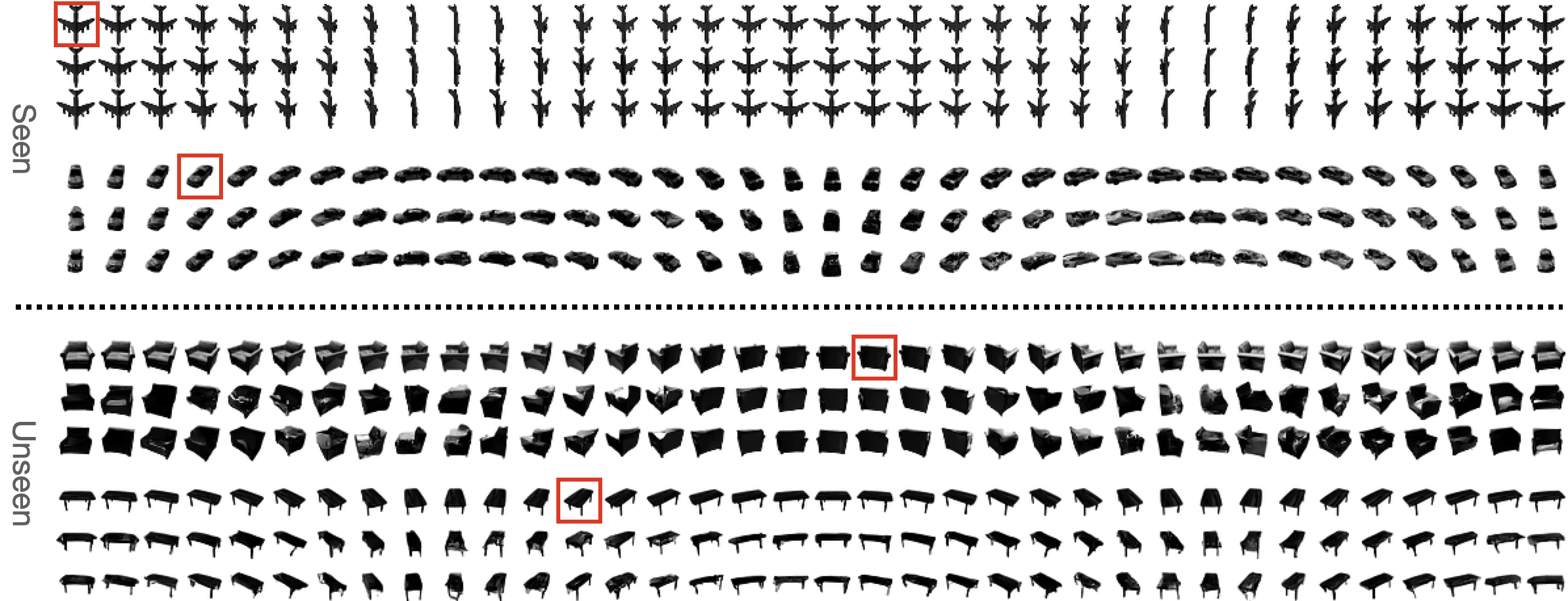}
    \vspace{-18pt}
    \caption{Neural processes on ShapeNet. First row: ground truth, red boxes indicate the context. Second row: predicted views given the context. Third row: predicted views from unseen angles.}
    \label{fig:process}
    \vspace{-5pt}
\end{figure}

The conditional version of \method can be viewed as a neural process which learns a distribution over functions. Instead of modeling low-dimensional functions, we model a process over images here. We evaluate on ShapeNet dataset \cite{chang2015shapenet}, which is constructed by viewing the objects from different angles. Our \method model takes several images from random angles as context and predicts the views for arbitrary angles. Figure \ref{fig:process} presents several examples for both seen and unseen categories from ShapeNet. We can see our \method model generates sharp images and smoothly transits between different viewpoints given just one context view. Our model can also generalize to unseen categories. 
\begin{wraptable}{r}{0.5\linewidth}
    \centering
    \vspace{-15pt}
    \caption{Bpd for generating 10 views given one random view.}
    \label{tab:process}
    \small
    \begin{tabular}{ccc}
    \toprule
         & Seen & Unseen \\
    \midrule
        cBRUNO & 1.43 & 1.62 \\
        \method & \textbf{1.34} & \textbf{1.41} \\
    \bottomrule
    \end{tabular}
\end{wraptable}
Conditional BRUNO \cite{korshunova2020conditional} trained with the same dataset sometimes generates images not in the same class as the specified context, while \method generation always matches with the context classes. Please see Fig.~\ref{fig:cbruno} for additional examples. In Table \ref{tab:process}, we report the bits per dimension (bpd) for generating a sequence of views given one context. Our model achieves lower bpd on both seen and unseen categories.

\begin{figure}
    \centering
    \subfigure[Occlusion removal]{
    \includegraphics[width=0.52\linewidth]{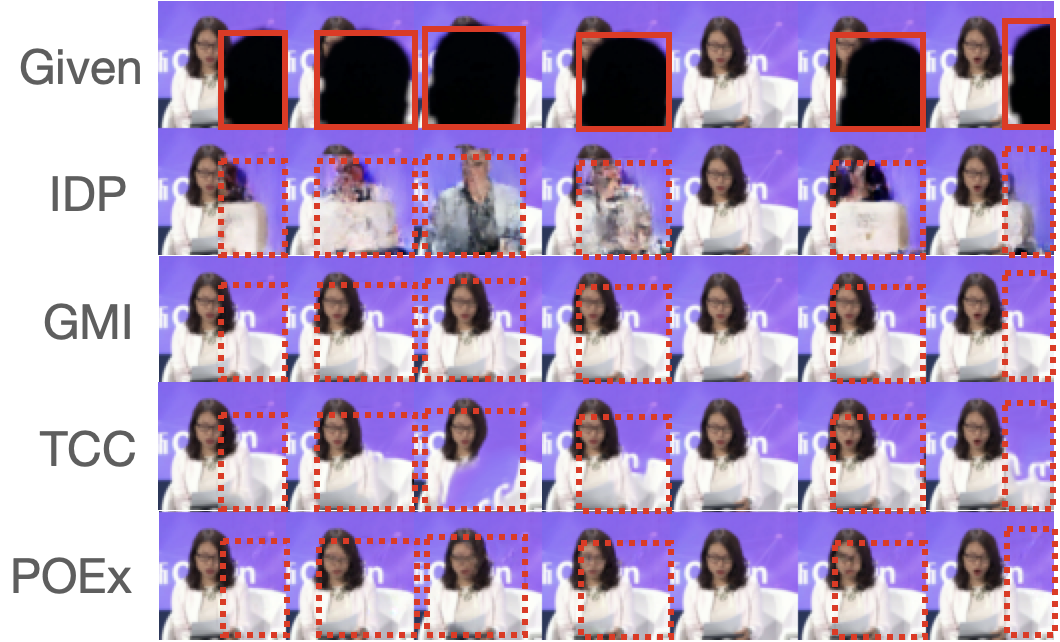}}
    \subfigure[Youtube]{\includegraphics[width=0.443\linewidth]{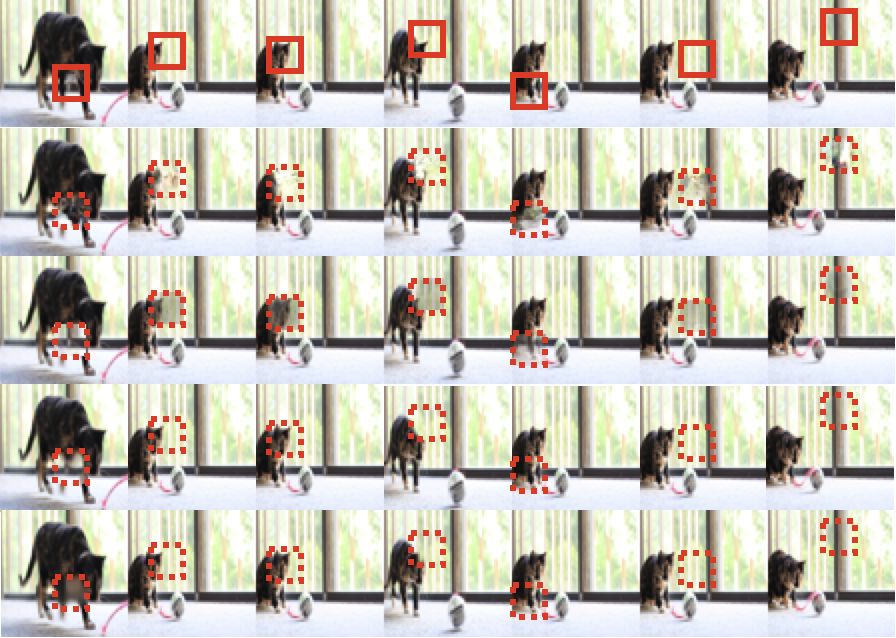}}
    \vspace{-10pt}
    \caption{Video inpainting. Better viewed with zoom-in.}
    \label{fig:video}
    \vspace{-15pt}
\end{figure}

\begin{wraptable}{r}{0.56\linewidth}
    \centering
    \vspace{-18pt}
    \caption{Video inpainting.}
    \label{tab:video}
    \tiny
    \begin{tabular}{c|cc|cc}
    \toprule
         & \multicolumn{2}{c|}{Occlusion} & \multicolumn{2}{c}{Youtube}\\
         & PSNR & SSIM & PSNR & SSIM \\
    \midrule
        IDP & 15.01 & 0.77 & 15.10 & 0.95 \\
        GMI & 19.85 & 0.82  & 16.49 & 0.96 \\
        TCC & \textbf{31.35} & 0.84 & \textbf{30.18} & 0.98 \\
        \method & 21.69 & \textbf{0.92} & 21.62 & \textbf{0.99} \\
    \bottomrule
    \end{tabular}
    \vspace{-5pt}
\end{wraptable}

With a conditional version of \methodnoindent, we can consider a collection of video frames conditioned on their timestamps as a set. Figure \ref{fig:video} shows the inpainting results on two video datasets from \citet{liao2020occlusion} and \citet{xu2018youtube}, and Table \ref{tab:video} reports the quantitative results. In addition to IDP, we compare to group mean imputation (GMI) \cite{sim2015missing} and TCC \cite{huang2016temporally}, which utilizes optical flow to infer the correspondence between adjacent frames. We can see \method outperforms IDP and GMI. There is still room for improvement with the help of optical flow, but we leave it for future works. GMI works well only if the content in the video does not move much. TCC does not work when the missing rate is high due to the difficulty of estimating the optical flow. Please see Fig.~\ref{fig:video_supp} and \ref{fig:youtube_supp} for additional examples.

\begin{figure}
    \centering
    \subfigure[Multi-task Gaussian processes]{
    \includegraphics[width=\linewidth]{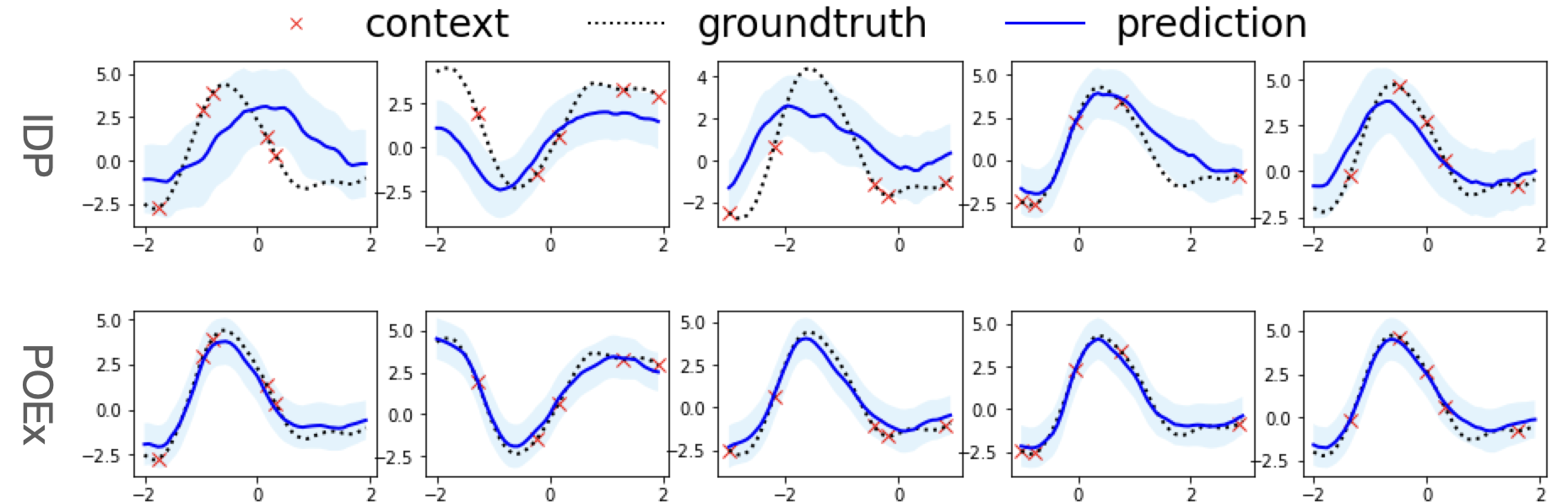}}
    \subfigure[MNIST with 50 context points]{\includegraphics[width=\linewidth]{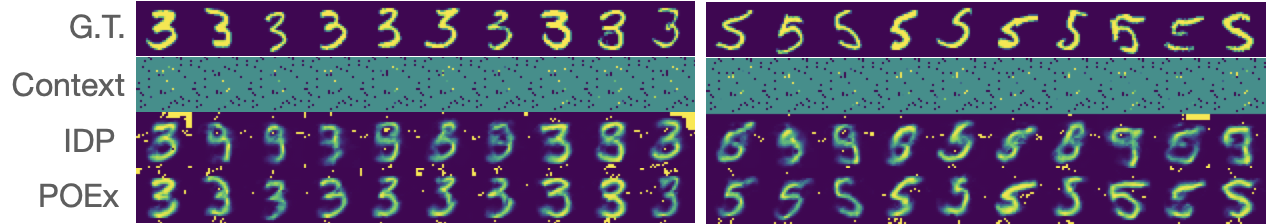}}
    \vspace{-10pt}
    \caption{Modeling a set of functions.}
    \label{fig:functions}
    \vspace{-15pt}
\end{figure}

Further generalizing to the infinite dimensional set elements, we propose to model a set of functions using our \method model. Similar to Neural Processes, we evaluate on Gaussian processes and simulated functions from images. 
\begin{wraptable}{r}{0.48\linewidth}
    \centering
    \vspace{-15pt}
    \caption{NLL for modeling a set of functions.}
    \label{tab:functions}
    \small
    \begin{tabular}{ccc}
    \toprule
         &  MTGP & MNIST \\
    \midrule
       IDP  & 2.04 & -1.08 \\
       \method & \textbf{1.79} & \textbf{-1.10}\\
     \bottomrule
    \end{tabular}
    \vspace{-5pt}
\end{wraptable}
Here, we use multi-task Gaussian processes \cite{bonilla2008multi}. For functions based on images, a set of MNIST images from the same class is used so that the set of functions are correlated. Figure \ref{fig:functions} present examples of modeling a set of correlated functions. We can see our \method model manages to recover the processes with low uncertainty using just a few context points, while the IDP model that treat each element independently fails. Table \ref{tab:functions} reports the negative log likelihood (NLL), and \method model obtains lower NLL on both datasets.

\section{Conclusion}
In this work, we develop the first model to work with sets of partially observed elements. Our \method model captures the intra-instance and inter-instance dependencies in a holistic framework by modeling the conditional distributions of the unobserved part conditioned on the observed part. We further reinterpret various applications as partially observed set modeling tasks and apply \method to solve them. \method is versatile and performs well for many challenging tasks even compared with domain specific approaches. For future works, we will explore domain specific architectures and techniques to further improve the performance.

\bibliography{main}

\begin{thebibliography}{46}
\providecommand{\natexlab}[1]{#1}
\providecommand{\url}[1]{\texttt{#1}}
\expandafter\ifx\csname urlstyle\endcsname\relax
  \providecommand{\doi}[1]{doi: #1}\else
  \providecommand{\doi}{doi: \begingroup \urlstyle{rm}\Url}\fi

\bibitem[Belghazi et~al.(2019)Belghazi, Oquab, and
  Lopez-Paz]{belghazi2019learning}
Belghazi, M., Oquab, M., and Lopez-Paz, D.
\newblock Learning about an exponential amount of conditional distributions.
\newblock In \emph{Advances in Neural Information Processing Systems}, pp.\
  13703--13714, 2019.

\bibitem[Bender et~al.(2020)Bender, O'Connor, Li, Garcia, Oliva, and
  Zaheer]{bender2020exchangeable}
Bender, C., O'Connor, K., Li, Y., Garcia, J., Oliva, J., and Zaheer, M.
\newblock Exchangeable generative models with flow scans.
\newblock In \emph{Proceedings of the AAAI Conference on Artificial
  Intelligence}, volume~34, pp.\  10053--10060, 2020.

\bibitem[Blei et~al.(2003)Blei, Ng, and Jordan]{blei2003latent}
Blei, D.~M., Ng, A.~Y., and Jordan, M.~I.
\newblock Latent dirichlet allocation.
\newblock \emph{Journal of machine Learning research}, 3\penalty0
  (Jan):\penalty0 993--1022, 2003.

\bibitem[Blei et~al.(2007)Blei, Lafferty, et~al.]{blei2007correlated}
Blei, D.~M., Lafferty, J.~D., et~al.
\newblock A correlated topic model of science.
\newblock \emph{The Annals of Applied Statistics}, 1\penalty0 (1):\penalty0
  17--35, 2007.

\bibitem[Bonilla et~al.(2008)Bonilla, Chai, and Williams]{bonilla2008multi}
Bonilla, E.~V., Chai, K.~M., and Williams, C.
\newblock Multi-task gaussian process prediction.
\newblock In \emph{Advances in neural information processing systems}, pp.\
  153--160, 2008.

\bibitem[Butz et~al.(2019)Butz, Oliveira, dos Santos, and
  Teixeira]{butz2019deep}
Butz, C.~J., Oliveira, J.~S., dos Santos, A.~E., and Teixeira, A.~L.
\newblock Deep convolutional sum-product networks.
\newblock In \emph{Proceedings of the AAAI Conference on Artificial
  Intelligence}, volume~33, pp.\  3248--3255, 2019.

\bibitem[Chang et~al.(2015)Chang, Funkhouser, Guibas, Hanrahan, Huang, Li,
  Savarese, Savva, Song, Su, et~al.]{chang2015shapenet}
Chang, A.~X., Funkhouser, T., Guibas, L., Hanrahan, P., Huang, Q., Li, Z.,
  Savarese, S., Savva, M., Song, S., Su, H., et~al.
\newblock Shapenet: An information-rich 3d model repository.
\newblock \emph{arXiv preprint arXiv:1512.03012}, 2015.

\bibitem[Douglas et~al.(2017)Douglas, Zarov, Gourgoulias, Lucas, Hart, Baker,
  Sahani, Perov, and Johri]{douglas2017universal}
Douglas, L., Zarov, I., Gourgoulias, K., Lucas, C., Hart, C., Baker, A.,
  Sahani, M., Perov, Y., and Johri, S.
\newblock A universal marginalizer for amortized inference in generative
  models.
\newblock \emph{arXiv preprint arXiv:1711.00695}, 2017.

\bibitem[Edwards \& Storkey(2016)Edwards and Storkey]{edwards2016towards}
Edwards, H. and Storkey, A.
\newblock Towards a neural statistician.
\newblock \emph{arXiv preprint arXiv:1606.02185}, 2016.

\bibitem[Finn et~al.(2017)Finn, Abbeel, and Levine]{finn2017model}
Finn, C., Abbeel, P., and Levine, S.
\newblock Model-agnostic meta-learning for fast adaptation of deep networks.
\newblock In \emph{International Conference on Machine Learning}, pp.\
  1126--1135. PMLR, 2017.

\bibitem[Garnelo et~al.(2018{\natexlab{a}})Garnelo, Rosenbaum, Maddison,
  Ramalho, Saxton, Shanahan, Teh, Rezende, and Eslami]{garnelo2018conditional}
Garnelo, M., Rosenbaum, D., Maddison, C.~J., Ramalho, T., Saxton, D., Shanahan,
  M., Teh, Y.~W., Rezende, D.~J., and Eslami, S.
\newblock Conditional neural processes.
\newblock \emph{arXiv preprint arXiv:1807.01613}, 2018{\natexlab{a}}.

\bibitem[Garnelo et~al.(2018{\natexlab{b}})Garnelo, Schwarz, Rosenbaum, Viola,
  Rezende, Eslami, and Teh]{garnelo2018neural}
Garnelo, M., Schwarz, J., Rosenbaum, D., Viola, F., Rezende, D.~J., Eslami, S.,
  and Teh, Y.~W.
\newblock Neural processes.
\newblock \emph{arXiv preprint arXiv:1807.01622}, 2018{\natexlab{b}}.

\bibitem[Griffiths \& Steyvers(2004)Griffiths and
  Steyvers]{griffiths2004finding}
Griffiths, T.~L. and Steyvers, M.
\newblock Finding scientific topics.
\newblock \emph{Proceedings of the National academy of Sciences}, 101\penalty0
  (suppl 1):\penalty0 5228--5235, 2004.

\bibitem[Heinze-Deml et~al.(2018)Heinze-Deml, Maathuis, and
  Meinshausen]{heinze2018causal}
Heinze-Deml, C., Maathuis, M.~H., and Meinshausen, N.
\newblock Causal structure learning.
\newblock \emph{Annual Review of Statistics and Its Application}, 5:\penalty0
  371--391, 2018.

\bibitem[Huang et~al.(2016)Huang, Kang, Ahuja, and Kopf]{huang2016temporally}
Huang, J.-B., Kang, S.~B., Ahuja, N., and Kopf, J.
\newblock Temporally coherent completion of dynamic video.
\newblock \emph{ACM Transactions on Graphics (TOG)}, 35\penalty0 (6):\penalty0
  1--11, 2016.

\bibitem[Ivanov et~al.(2018)Ivanov, Figurnov, and
  Vetrov]{ivanov2018variational}
Ivanov, O., Figurnov, M., and Vetrov, D.
\newblock Variational autoencoder with arbitrary conditioning.
\newblock In \emph{International Conference on Learning Representations}, 2018.

\bibitem[Jaini et~al.(2018)Jaini, Poupart, and Yu]{jaini2018deep}
Jaini, P., Poupart, P., and Yu, Y.
\newblock Deep homogeneous mixture models: representation, separation, and
  approximation.
\newblock In \emph{Advances in Neural Information Processing Systems}, pp.\
  7136--7145, 2018.

\bibitem[Kim et~al.(2019)Kim, Mnih, Schwarz, Garnelo, Eslami, Rosenbaum,
  Vinyals, and Teh]{kim2019attentive}
Kim, H., Mnih, A., Schwarz, J., Garnelo, M., Eslami, A., Rosenbaum, D.,
  Vinyals, O., and Teh, Y.~W.
\newblock Attentive neural processes.
\newblock \emph{arXiv preprint arXiv:1901.05761}, 2019.

\bibitem[Korshunova et~al.(2018)Korshunova, Degrave, Husz{\'a}r, Gal, Gretton,
  and Dambre]{korshunova2018bruno}
Korshunova, I., Degrave, J., Husz{\'a}r, F., Gal, Y., Gretton, A., and Dambre,
  J.
\newblock Bruno: A deep recurrent model for exchangeable data.
\newblock \emph{Advances in Neural Information Processing Systems},
  31:\penalty0 7190--7198, 2018.

\bibitem[Korshunova et~al.(2020)Korshunova, Gal, Gretton, and
  Dambre]{korshunova2020conditional}
Korshunova, I., Gal, Y., Gretton, A., and Dambre, J.
\newblock Conditional bruno: A neural process for exchangeable labelled data.
\newblock \emph{Neurocomputing}, 2020.

\bibitem[Koster et~al.(2002)]{koster2002marginalizing}
Koster, J.~T. et~al.
\newblock Marginalizing and conditioning in graphical models.
\newblock \emph{Bernoulli}, 8\penalty0 (6):\penalty0 817--840, 2002.

\bibitem[Lee et~al.(2019)Lee, Lee, Kim, Kosiorek, Choi, and Teh]{lee2019set}
Lee, J., Lee, Y., Kim, J., Kosiorek, A., Choi, S., and Teh, Y.~W.
\newblock Set transformer: A framework for attention-based
  permutation-invariant neural networks.
\newblock In \emph{International Conference on Machine Learning}, pp.\
  3744--3753. PMLR, 2019.

\bibitem[Li et~al.(2020{\natexlab{a}})Li, Akbar, and Oliva]{li2019flow}
Li, Y., Akbar, S., and Oliva, J.
\newblock {ACF}low: Flow models for arbitrary conditional likelihoods.
\newblock In \emph{Proceedings of the 37th International Conference on Machine
  Learning}, 2020{\natexlab{a}}.

\bibitem[Li et~al.(2020{\natexlab{b}})Li, Yi, Bender, Shan, and
  Oliva]{li2020exchangeable}
Li, Y., Yi, H., Bender, C., Shan, S., and Oliva, J.~B.
\newblock Exchangeable neural ode for set modeling.
\newblock \emph{Advances in Neural Information Processing Systems}, 33,
  2020{\natexlab{b}}.

\bibitem[Liao et~al.(2020)Liao, Duan, Li, Xu, Yang, Cai, Chen, and
  Chen]{liao2020occlusion}
Liao, J., Duan, H., Li, X., Xu, H., Yang, Y., Cai, W., Chen, Y., and Chen, L.
\newblock Occlusion detection for automatic video editing.
\newblock In \emph{Proceedings of the 28th ACM International Conference on
  Multimedia}, pp.\  2255--2263, 2020.

\bibitem[{\O}ksendal(2003)]{oksendal2003stochastic}
{\O}ksendal, B.
\newblock Stochastic differential equations.
\newblock In \emph{Stochastic differential equations}, pp.\  65--84. Springer,
  2003.

\bibitem[Poon \& Domingos(2011)Poon and Domingos]{poon2011sum}
Poon, H. and Domingos, P.
\newblock Sum-product networks: A new deep architecture.
\newblock In \emph{2011 IEEE International Conference on Computer Vision
  Workshops (ICCV Workshops)}, pp.\  689--690. IEEE, 2011.

\bibitem[Porteous et~al.(2008)Porteous, Newman, Ihler, Asuncion, Smyth, and
  Welling]{porteous2008fast}
Porteous, I., Newman, D., Ihler, A., Asuncion, A., Smyth, P., and Welling, M.
\newblock Fast collapsed gibbs sampling for latent dirichlet allocation.
\newblock In \emph{Proceedings of the 14th ACM SIGKDD international conference
  on Knowledge discovery and data mining}, pp.\  569--577, 2008.

\bibitem[Qi et~al.(2017)Qi, Su, Mo, and Guibas]{qi2017pointnet}
Qi, C.~R., Su, H., Mo, K., and Guibas, L.~J.
\newblock Pointnet: Deep learning on point sets for 3d classification and
  segmentation.
\newblock In \emph{Proceedings of the IEEE conference on computer vision and
  pattern recognition}, pp.\  652--660, 2017.

\bibitem[Rasmussen(2003)]{rasmussen2003gaussian}
Rasmussen, C.~E.
\newblock Gaussian processes in machine learning.
\newblock In \emph{Summer School on Machine Learning}, pp.\  63--71. Springer,
  2003.

\bibitem[Scutari et~al.(2019)Scutari, Graafland, and
  Guti{\'e}rrez]{scutari2019learns}
Scutari, M., Graafland, C.~E., and Guti{\'e}rrez, J.~M.
\newblock Who learns better bayesian network structures: Accuracy and speed of
  structure learning algorithms.
\newblock \emph{International Journal of Approximate Reasoning}, 115:\penalty0
  235--253, 2019.

\bibitem[Shah et~al.(2014)Shah, Wilson, and Ghahramani]{shah2014student}
Shah, A., Wilson, A., and Ghahramani, Z.
\newblock Student-t processes as alternatives to gaussian processes.
\newblock In \emph{Artificial intelligence and statistics}, pp.\  877--885,
  2014.

\bibitem[Sim et~al.(2015)Sim, Lee, and Kwon]{sim2015missing}
Sim, J., Lee, J.~S., and Kwon, O.
\newblock Missing values and optimal selection of an imputation method and
  classification algorithm to improve the accuracy of ubiquitous computing
  applications.
\newblock \emph{Mathematical problems in engineering}, 2015.

\bibitem[Tan \& Peharz(2019)Tan and Peharz]{tan2019hierarchical}
Tan, P.~L. and Peharz, R.
\newblock Hierarchical decompositional mixtures of variational autoencoders.
\newblock In \emph{International Conference on Machine Learning}, pp.\
  6115--6124, 2019.

\bibitem[Teh et~al.(2006{\natexlab{a}})Teh, Newman, and
  Welling]{teh2006collapsed}
Teh, Y., Newman, D., and Welling, M.
\newblock A collapsed variational bayesian inference algorithm for latent
  dirichlet allocation.
\newblock \emph{Advances in neural information processing systems},
  19:\penalty0 1353--1360, 2006{\natexlab{a}}.

\bibitem[Teh et~al.(2006{\natexlab{b}})Teh, Jordan, Beal, and
  Blei]{teh2006hierarchical}
Teh, Y.~W., Jordan, M.~I., Beal, M.~J., and Blei, D.~M.
\newblock Hierarchical dirichlet processes.
\newblock \emph{Journal of the american statistical association}, 101\penalty0
  (476):\penalty0 1566--1581, 2006{\natexlab{b}}.

\bibitem[Vaswani et~al.(2017)Vaswani, Shazeer, Parmar, Uszkoreit, Jones, Gomez,
  Kaiser, and Polosukhin]{vaswani2017attention}
Vaswani, A., Shazeer, N., Parmar, N., Uszkoreit, J., Jones, L., Gomez, A.~N.,
  Kaiser, L., and Polosukhin, I.
\newblock Attention is all you need.
\newblock \emph{arXiv preprint arXiv:1706.03762}, 2017.

\bibitem[Vinyals et~al.(2015)Vinyals, Bengio, and Kudlur]{vinyals2015order}
Vinyals, O., Bengio, S., and Kudlur, M.
\newblock Order matters: Sequence to sequence for sets.
\newblock \emph{arXiv preprint arXiv:1511.06391}, 2015.

\bibitem[Wang et~al.(2020)Wang, Liu, Yue, Lasenby, and Kusner]{wang2020pre}
Wang, H., Liu, Q., Yue, X., Lasenby, J., and Kusner, M.~J.
\newblock Pre-training by completing point clouds.
\newblock \emph{arXiv preprint arXiv:2010.01089}, 2020.

\bibitem[Wang et~al.(2017)Wang, Aggarwal, and Aeron]{wang2017efficient}
Wang, W., Aggarwal, V., and Aeron, S.
\newblock Efficient low rank tensor ring completion.
\newblock In \emph{Proceedings of the IEEE International Conference on Computer
  Vision}, pp.\  5697--5705, 2017.

\bibitem[Xu et~al.(2018)Xu, Yang, Fan, Yue, Liang, Yang, and
  Huang]{xu2018youtube}
Xu, N., Yang, L., Fan, Y., Yue, D., Liang, Y., Yang, J., and Huang, T.
\newblock Youtube-vos: A large-scale video object segmentation benchmark.
\newblock \emph{arXiv preprint arXiv:1809.03327}, 2018.

\bibitem[Yang et~al.(2019)Yang, Huang, Hao, Liu, Belongie, and
  Hariharan]{yang2019pointflow}
Yang, G., Huang, X., Hao, Z., Liu, M.-Y., Belongie, S., and Hariharan, B.
\newblock Pointflow: 3d point cloud generation with continuous normalizing
  flows.
\newblock In \emph{Proceedings of the IEEE International Conference on Computer
  Vision}, pp.\  4541--4550, 2019.

\bibitem[Yang et~al.(2020)Yang, Dai, Dai, and Schuurmans]{yang2020energy}
Yang, M., Dai, B., Dai, H., and Schuurmans, D.
\newblock Energy-based processes for exchangeable data.
\newblock \emph{arXiv preprint arXiv:2003.07521}, 2020.

\bibitem[Yang et~al.(2018)Yang, Feng, Shen, and Tian]{yang2018foldingnet}
Yang, Y., Feng, C., Shen, Y., and Tian, D.
\newblock Foldingnet: Point cloud auto-encoder via deep grid deformation.
\newblock In \emph{Proceedings of the IEEE Conference on Computer Vision and
  Pattern Recognition}, pp.\  206--215, 2018.

\bibitem[Yu et~al.(2018)Yu, Li, Fu, Cohen-Or, and Heng]{yu2018pu}
Yu, L., Li, X., Fu, C.-W., Cohen-Or, D., and Heng, P.-A.
\newblock Pu-net: Point cloud upsampling network.
\newblock In \emph{Proceedings of the IEEE Conference on Computer Vision and
  Pattern Recognition}, pp.\  2790--2799, 2018.

\bibitem[Yuan et~al.(2018)Yuan, Khot, Held, Mertz, and Hebert]{yuan2018pcn}
Yuan, W., Khot, T., Held, D., Mertz, C., and Hebert, M.
\newblock Pcn: Point completion network.
\newblock In \emph{2018 International Conference on 3D Vision (3DV)}, pp.\
  728--737. IEEE, 2018.

\end{thebibliography}
\bibliographystyle{icml2021}

\clearpage
\appendix

\setcounter{theorem}{0}
\setcounter{figure}{0}
\setcounter{table}{0}
\renewcommand\thefigure{\thesection.\arabic{figure}}
\renewcommand\thetable{\thesection.\arabic{table}}

\section{Proof}

\begin{theorem}
Given a set of observations $\mathbf{x}=\{x_i\}_{i=1}^{N}$ from an infinitely exchangeable process, denote the observed and unobserved part as $\mathbf{x}_o = \{x_i^{(o_i)}\}_{i=1}^{N}$ and $\mathbf{x}_u = \{x_i^{(u_i)}\}_{i=1}^{N}$ respectively. Then the arbitrary conditional distribution $p(\mathbf{x}_u \mid \mathbf{x}_o)$ can be decomposed as follows:
\begin{equation*}
    p(\mathbf{x}_u \mid \mathbf{x}_o) = \int \prod_{i=1}^{N} p(x_i^{(u_i)} \mid x_i^{(o_i)}, \theta) p(\theta \mid \mathbf{x}_o) d\theta.
\end{equation*}
\end{theorem}

\begin{proof}
According to the definition, we have $u_i,o_i \subseteq \{1,\ldots,d\}$. Define $m_i = u_i \cup o_i$ and $\mathbf{x}_m = \mathbf{x}_u \cup \mathbf{x}_o = \{x_i^{(m_i)}\}_{i=1}^{N}$, where $\mathbf{x}_m$ represents the set of features that contains both observed and unobserved dimensions.

From the De Finetti’s theorem, we can derive the following equations:
\begin{equation*}
\begin{aligned}
    p(\mathbf{x}_m) &= \int \prod_{i=1}^{N} p(x_i^{(m_i)} \mid \theta) p(\theta) d\theta \\
    &= \int \prod_{i=1}^{N} \left[ p(x_i^{(u_i)} \mid x_i^{(o_i)}, \theta) p(x_i^{(o_i)} \mid \theta) \right] p(\theta) d\theta \\
    &= \int \prod_{i=1}^{N} p(x_i^{(u_i)} \mid x_i^{(o_i)}, \theta) \prod_{i=1}^{N} p(x_i^{(o_i)} \mid \theta) p(\theta) d\theta \\
    &= \int \prod_{i=1}^{N} p(x_i^{(u_i)} \mid x_i^{(o_i)}, \theta) p(\mathbf{x}_o \mid \theta) p(\theta) d\theta \\
    &= \int \prod_{i=1}^{N} p(x_i^{(u_i)} \mid x_i^{(o_i)}, \theta) p(\theta \mid \mathbf{x}_o) p(\mathbf{x}_o) d\theta
\end{aligned}
\end{equation*}
The key step is the penultimate equation, where the De Finetti’s theorem is applied for $\mathbf{x}_o$ in reverse direction. That is, given the same latent code, we assume set elements of $\mathbf{x}_o$ are conditionally i.i.d.. Since $\mathbf{x}_o$ contains subsets of features from $\mathbf{x}_m$, the above assumption holds.

Dividing both side of the equation by $p(\mathbf{x}_o)$ gives
\begin{equation*}
    p(\mathbf{x}_u \mid \mathbf{x}_o) = \int \prod_{i=1}^{N} p(x_i^{(u_i)} \mid x_i^{(o_i)}, \theta) p(\theta \mid \mathbf{x}_o) d\theta
\end{equation*}

\end{proof}

\section{Models}\label{sec:nets}
As shown in Fig.~\ref{fig:vae}, our model mainly contains 4 parts: posterior, prior, permutation equivariant embedding and the generator. The prior can be further divided as base distribution (B) and the flow transformations (Q). Here, we describe the specific architectures used for each component respectively. Please see Table \ref{tab:net} for details. Note we did not tune the network architecture heavily. Further improvement is expected by tuning the network for each task separately.

For simplicity, the posterior and prior are mostly constructed by processing each set element independently then pooling across the set. A Gaussian distribution is then derived from the permutation invariant embedding. For point clouds, we use set transformer to incorporate dependencies among points. For set of functions (each function is represented as a set of (input, target) pairs), we first use set transformer to get an permutation invariant embedding for each function. Then, we take the average pooling over the set as the embedding for the set. The prior utilizes a normalizing flow to increase the flexibility, which is implemented as a stack of invertible transformations over the samples. For images, we use a multi-scale ACFlow as the generator, which is similar to the original model used in \cite{li2019flow}. We modify it to a conditional version so that both the base distribution and the transformations are conditioned on the given embeddings. For point cloud, we use the conditional ACFlow with an autoregressive base likelihood.

\begin{table*}[h]
    \centering
    \caption{Network architectures}
    \label{tab:net}
    \tiny
    \begin{tabular}{c|c|c}
    \toprule
        Data Type                     & Components  & Architecture \\
    \midrule
       \multirow{5}{*}{Set of Images}  & Posterior & [Conv(128,3,1), Conv(128,3,1), MaxP(2,2)]$\times 4$ + FC(256) + SetAvgPool $\to$ Gaussian(128)\\
                                       & Prior (B) & [Conv(128,3,1), Conv(128,3,1), MaxP(2,2)]$\times 4$ + FC(256) + SetAvgPool $\to$ Gaussian(128)\\
                                       & Prior (Q) & [Linear, LeakyReLU, Affine Coupling, Permute]$\times 4$\\
                                       & Peq Embed & [Conv(64,3,1), Conv(64,3,1), MaxP(2,2)]$\times 2$ + [DecomAttn, ResBlock]$\times 4$ + [DeConv(64,3,2), Conv(64,3,1)]$\times 2$\\
                                       & Generator & Conditional ACFlow(multi-scale)\\
    \midrule
       \multirow{5}{*}{Point Clouds}   & Posterior & [SetTransformer(256)]$\times 4$ + SetAvgPool + FC(512) $\to$ Gaussian(256)\\
                                       & Prior (B) & [SetTransformer(256)]$\times 4$ + SetAvgPool + FC(512) $\to$ Gaussian(256)\\
                                       & Prior (Q) & [Linear, LeakyReLU, Affine Coupling, Permute]$\times 4$ \\
                                       & Peq Embed & [SetTransformer(256)]$\times 4$ \\
                                       & Generator & Conditional ACFlow\\
    \midrule
       \multirow{5}{*}{Set of Functions} & Posterior & [SetTransformer(256)]$\times 4$ + SetAvgPool + SetAvgPool + FC(512) $\to$ Gaussian(256) \\
                                         & Prior (B) & [SetTransformer(256)]$\times 4$ + SetAvgPool + SetAvgPool + FC(512) $\to$ Gaussian(256)\\
                                         & Prior (Q) & [Linear, LeakyReLU, Affine Coupling, Permute]$\times 4$\\
                                         & Peq Embed & SelfAttention + CrossAttention\\
                                         & Generator & [FC(256)]$\times 4$ + FC(2) $\to$ Gaussian(1)\\
    \bottomrule
    \end{tabular}
\end{table*}

\section{Experiments}
\subsection{Image Inpainting}
For image inpainting, we evaluate on both MNIST and Omniglot datasets. We construct the sets by randomly selecting 10 images from a certain class. The observed part for each image is a $10 \times 10$ square placed at random positions.

The baseline TRC \cite{wang2017efficient} is based on tensor ring completion. We use their official implementation \footnote{https://github.com/wangwenqi1990/TensorRingCompletion} and cross validate the hyperparameters on our datasets. The set of images with size $[32,32,10]$ are treated as a 3-order tensor and reshaped into a 10-order tensor of size $[4,2,2,2,4,2,2,2,2,5]$. The tensor is then completed by alternating least square method with a specified tensor ring rank (TR-Rank). We found the TR-Rank of 9 works best for our datasets.

Figure \ref{fig:inpainting_supp} presents several additional examples for inpainting a set of images. We can see the TRC fails to recover any meaningful structures, IDP fails to infer the right classes, while our \method model always generates the characters from the specified classes.

\begin{figure}
    \centering
    \subfigure[MNIST]{
    \includegraphics[width=\linewidth]{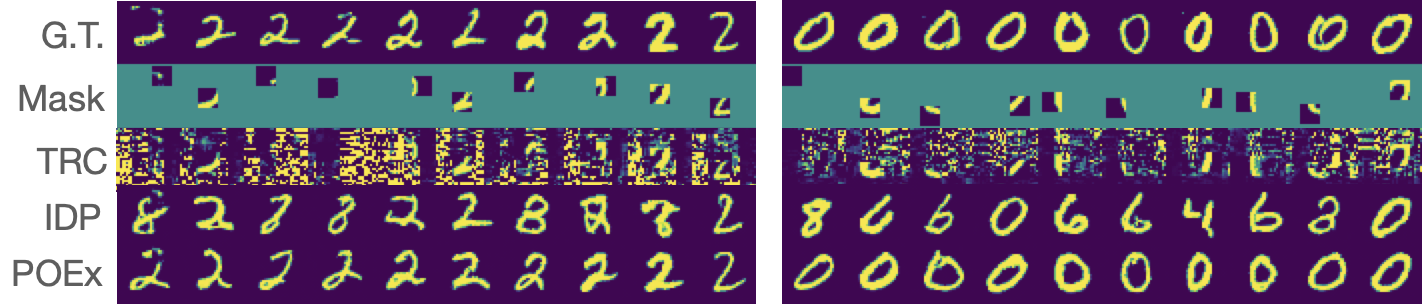}}
    \subfigure[Omniglot]{
    \includegraphics[width=\linewidth]{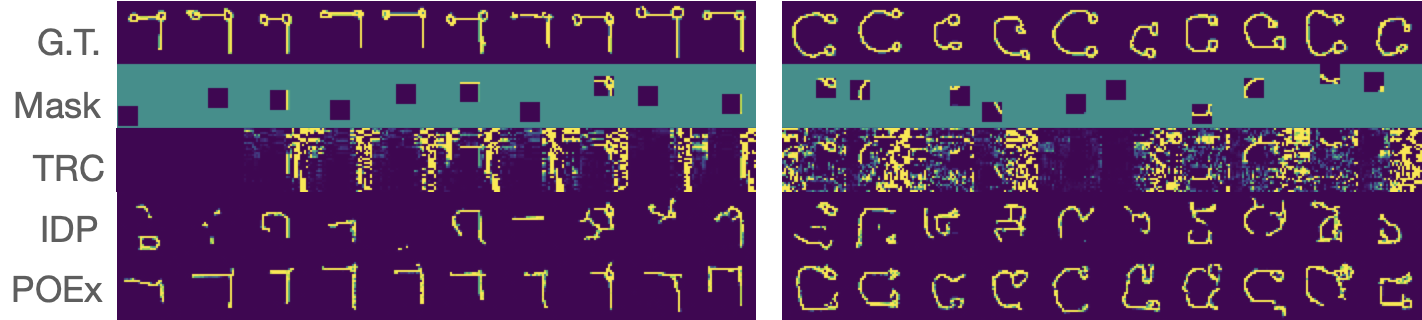}}
    \caption{Inpaint the missing values for a set of images.}
    \label{fig:inpainting_supp}
\end{figure}

\subsection{Image Set Expansion}
To expand a set, we modify the distribution of the masks so that some elements are fully observed and others are fully unobserved. We randomly select the number of observed elements during training. During test, our model can expand the set to arbitrary size.

\subsection{Few-shot Learning}
Since our model can expand sets for even unseen categories, we utilize it to augment the support set for few-shot learning. Given a N-way-K-shot training set, we use the K exemplars from each class to generate M novel instances, thus change the problem to N-way-(M+K)-shot classification. We train the MAML with fully connected networks using the official implementation \footnote{https://github.com/cbfinn/maml}.

\subsection{Point Cloud Completion}
For point cloud completion, we build the dataset by sampling 256 points from the observed part and 1792 points from the occluded part, thus there are 2048 points in total. PCN is trained with a multi-scale architecture, where the given point cloud is first expanded to 512 points and then expanded further to 2048 points. We use EMD loss and CD loss for the coarse and fine outputs respectively.

\begin{figure}
    \centering
    \includegraphics[width=\linewidth]{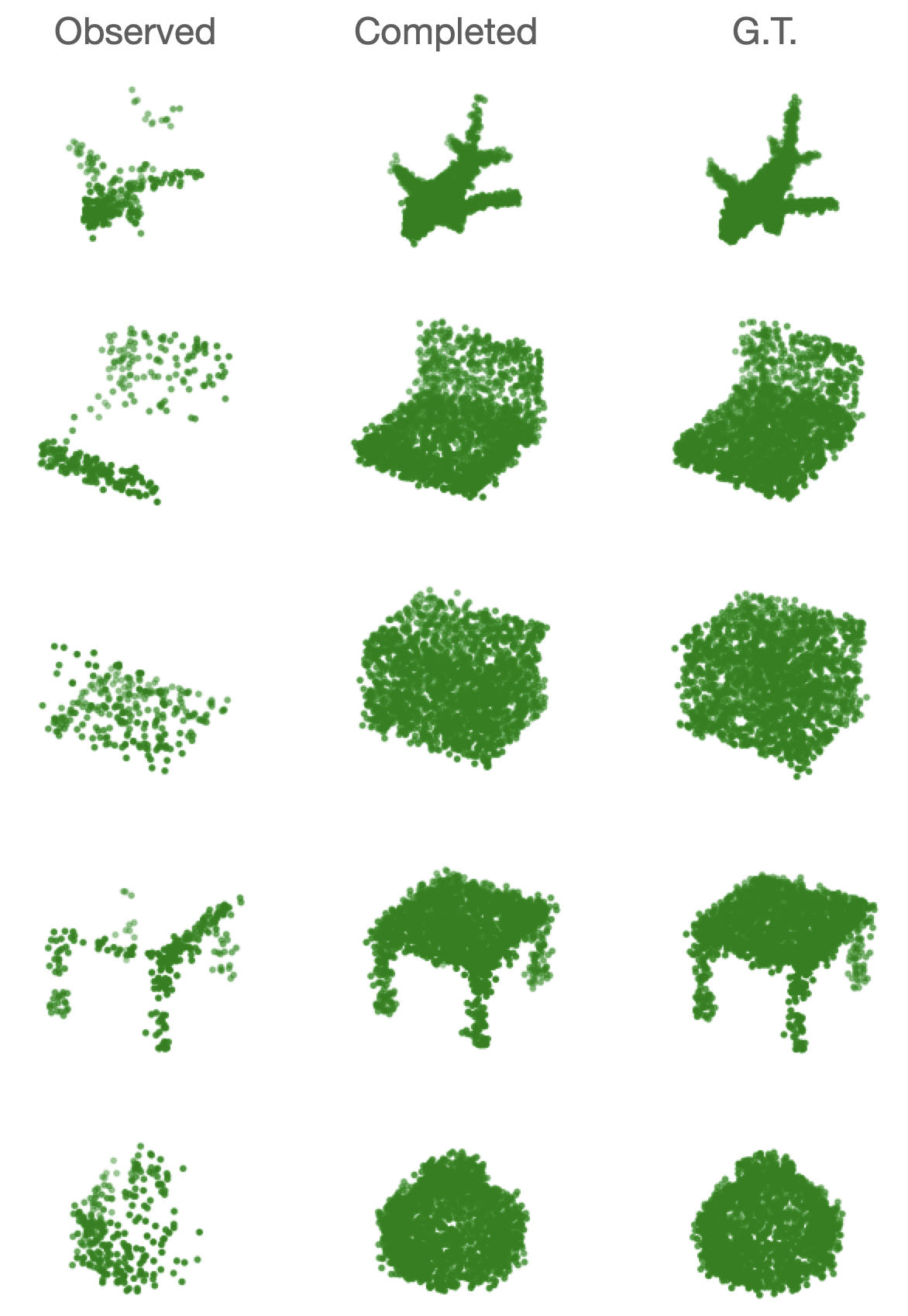}
    \caption{Additional examples for point cloud completion.}
    \label{fig:completion_supp}
\end{figure}

\subsection{Point Cloud Upsampling}
Point cloud upsampling uses the ModelNet40 dataset. We leave 10 categories out that are not used during training to evaluate the generalization ability of our model. We uniformaly sample 2048 points as the target. During training, an arbitrary subset is taken as input. PUNet is not built for upsampling arbitrary sized point cloud, therefore we subsample 256 points as input. PUNet is trained to optimize the EMD loss and a repulsion loss as in their original work. For comparison, we evaluate our \method model and PUNet for upsampling 256 points.

\begin{figure}
    \centering
    \includegraphics[width=\linewidth]{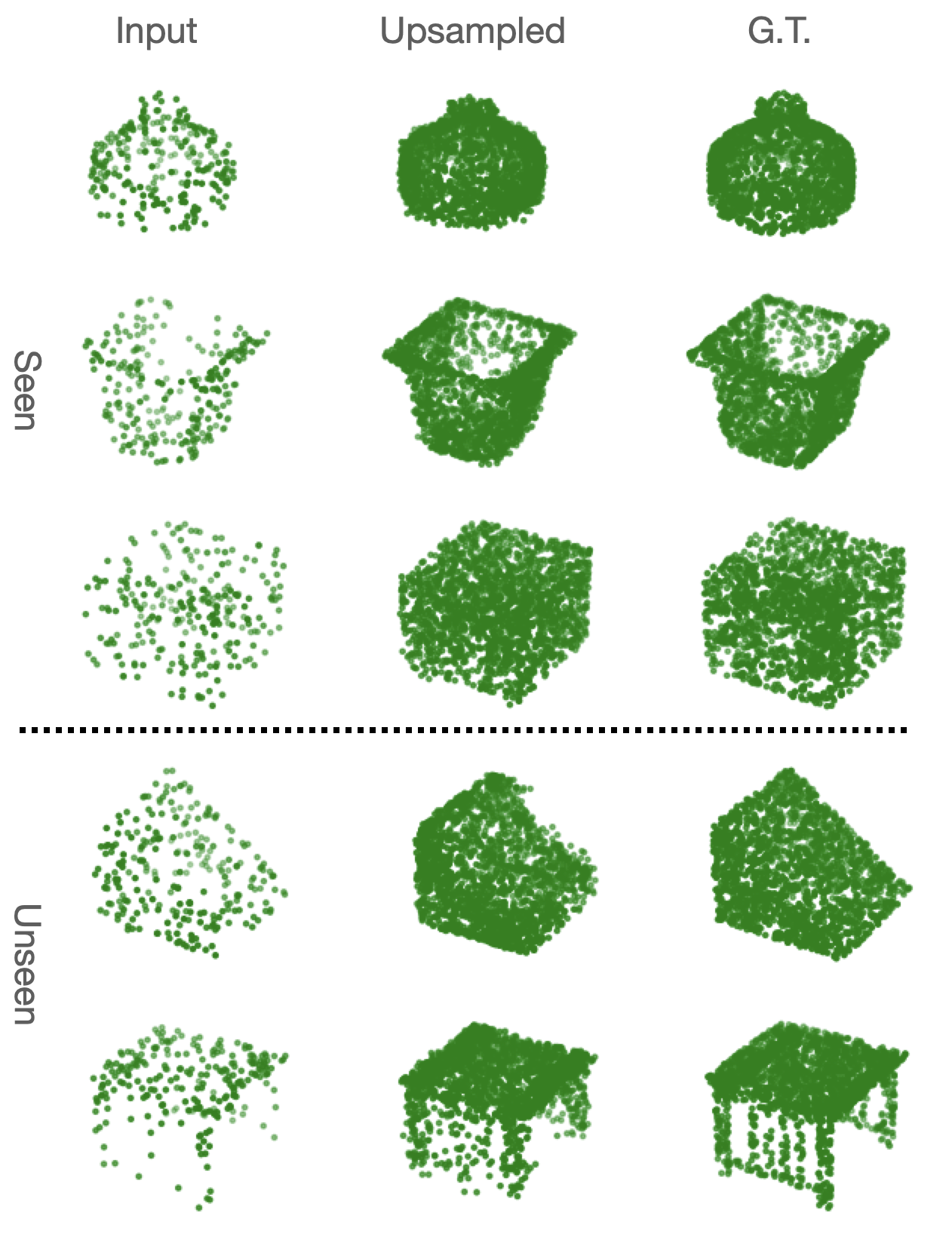}
    \caption{Additional examples for point cloud upsampling.}
    \label{fig:upsampling_supp}
\end{figure}

\subsection{Point Cloud Compression}
One advantage of our \method model is that the likelihood can be used to guide the subset selection. To evaluate the selection quality, we build a biased point cloud by sampling the center points with higher probability. Specifically, 2048 points are selected from the original point cloud with probability $\text{softmax}(\Vert x - m \Vert^2 / T^2)$, where $x$ and $m$ are the coordinates of each candidate points and the center respectively. We use $T=0.1$ for our experiments. Given 2048 non-uniformly sampled points, our goal is to sample 256 points from it to represent the underlying geometry. We propose a sequential selection strategy, where one point is selected per step based on its conditional likelihood. The one with lowest likelihood given the current selected points is selected at each step. Baseline approaches include uniform sampling, where 256 points are uniformly sampled. A k-means based sampling method group the given points into 256 clusters and we select one point from each cluster that is closest to its cluster center. The farthest point sampling algorithm also proceeds sequentially, where the point that is farthest from the current selected points is selected at each step. To quantitatively evaluate the selection quality we reconstruct the selected subset to 2048 points and compare it to a uniformly sampled point cloud. If the selected subset represents the geometry well, the upsampled point cloud should be close to the uniformly sampled one.

\subsection{Colonoscopy Point Cloud Imputation}
We uniformly sample 2048 points from the reconstructed colonoscopy meshes as our training data. During training, points inside a random ball are viewed as unobserved blind spot and our model is trained to impute those blind spots. To provide guidance for the imputation, we divide the space into small cubes and condition our model on the cube coordinates. That is, points inside the cube are indexed by the corresponding cube coordinates. To train the conditional \method model, the indexes and the point coordinates are concatenated together as inputs. Figure \ref{fig:colon_supp} presents several examples of the imputed point clouds.

\begin{figure}
    \centering
    \includegraphics[width=\linewidth]{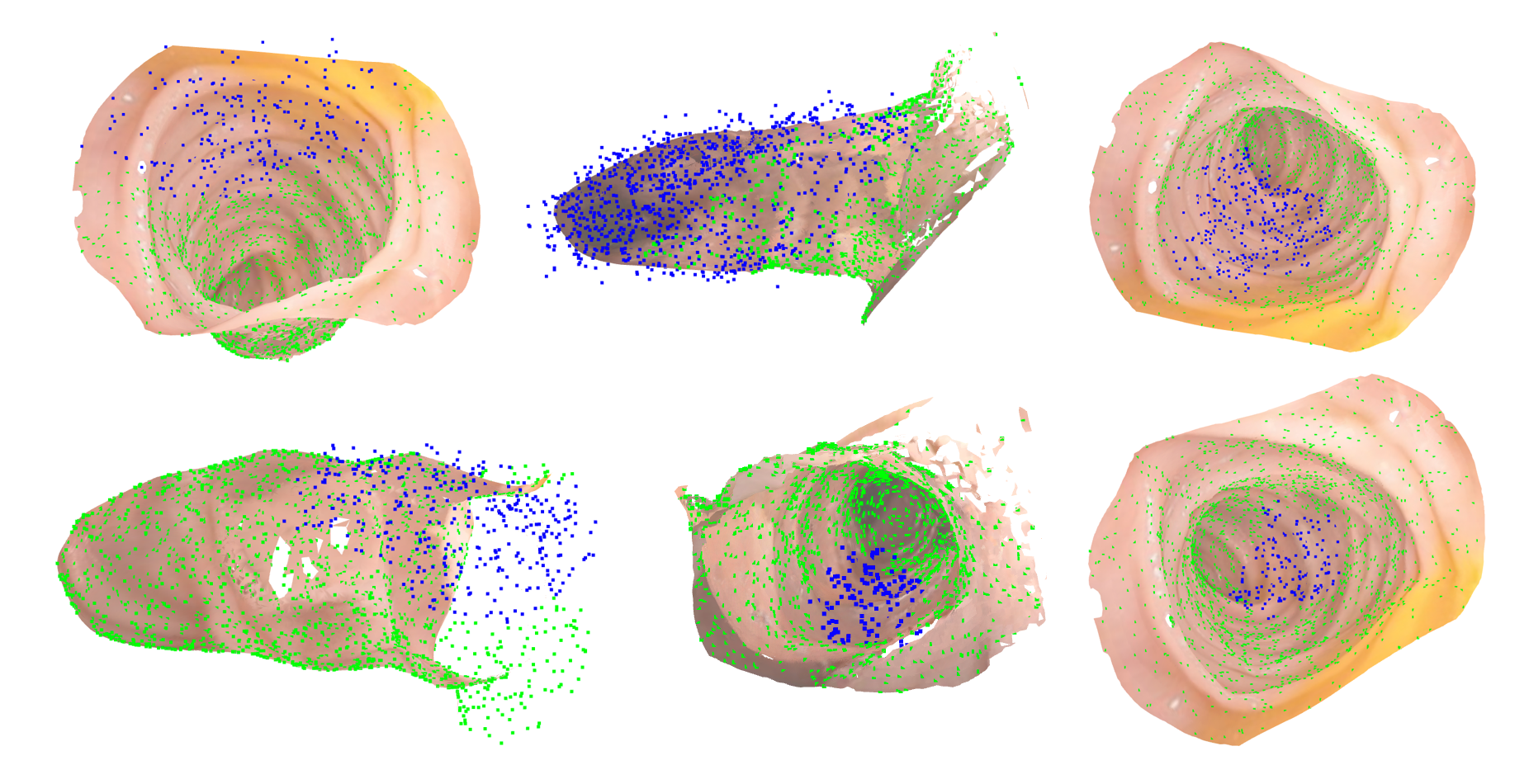}
    \vspace{-10pt}
    \caption{Additional examples for imputing colonoscopy point cloud.}
    \label{fig:colon_supp}
\end{figure}

\subsection{Neural Processes over Images}
The conditional version of \method can be interpreted as a neural process. We evaluate our model on the ShapeNet dataset, where images of a object viewed from different angles are used as the target variables. Given the context that contains images from several random views, the neural processes are expected to generate novel images for arbitrary view points. We split the dataset into seen and unseen categories and train our \method model and the conditional BRUNO only on seen categories. Images are indexed with the continuous angles, which we translate to their sin and cos values. Figure \ref{fig:cbruno} shows several examples of generated images from conditional BRUNO. We can see the generated images do not always match with the provided context.

\begin{figure}
    \centering
    \includegraphics[width=\linewidth]{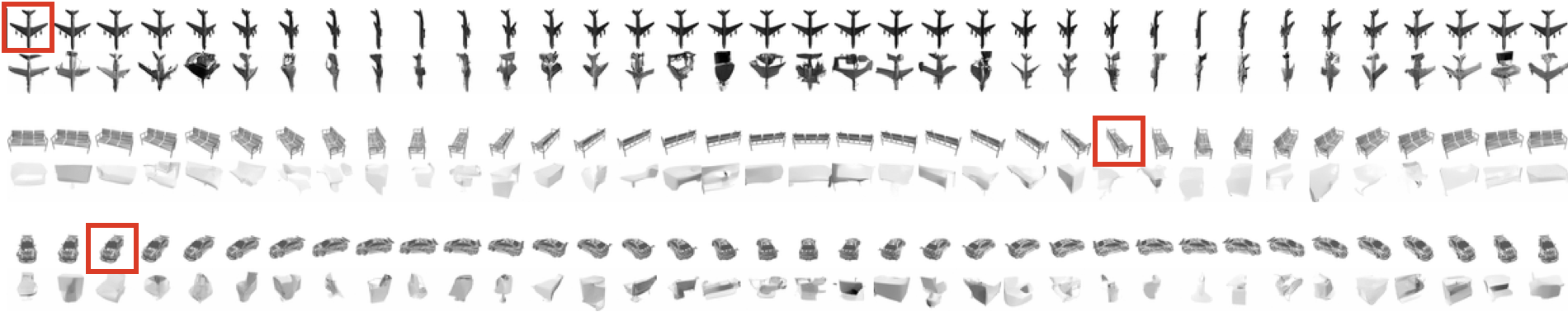}
    \vspace{-10pt}
    \caption{Neural Process sampling from conditional BRUNO. Red coxes indicate the given context. First row: ground truth. Second row: generation.}
    \label{fig:cbruno}
\end{figure}

\subsection{Video Inpainting}
We evaluate our model for video inpainting with two datasets. 10 frames are randomly selected from a video to construct a set. The frames are indexed by their timestamps. We normalize them to the range of [0,1] and further calculate a 32-dimensional positional embedding similar to \cite{vaswani2017attention}. To simulate occluded pixels, we put a $16 \times 16$ square at random position, pixels inside the square are considered missing. For occlusion removal, since we do not have the ground truth values for occluded pixels, those pixels are excluded for training and evaluation. We run TCC \cite{huang2016temporally} using the official code \footnote{https://github.com/amjltc295/Temporally-Coherent-Completion-of-Dynamic-Video} and their recommended hyperparameters. For group mean imputation, if a certain pixel is missing in all frames, we impute it using the global mean of all observed pixels. Figure \ref{fig:video_supp} and \ref{fig:youtube_supp} present additional examples for these two datasets. GMI works well if the movement is negligible, but it fails when the object moves too much. It also struggles if certain pixels are missing in all frames. In the second example of Fig.~\ref{fig:video_supp}, TCC fails to run since the missing rate is too high for certain frames and the optical flow cannot be properly estimated. For quantitative evaluation, we report the PSNR and SSIM scores between the ground truth and the imputed images. We calculate PSNR only for the missing part and the SSIM for the whole image.

\begin{figure}
    \centering
    \includegraphics[width=\linewidth]{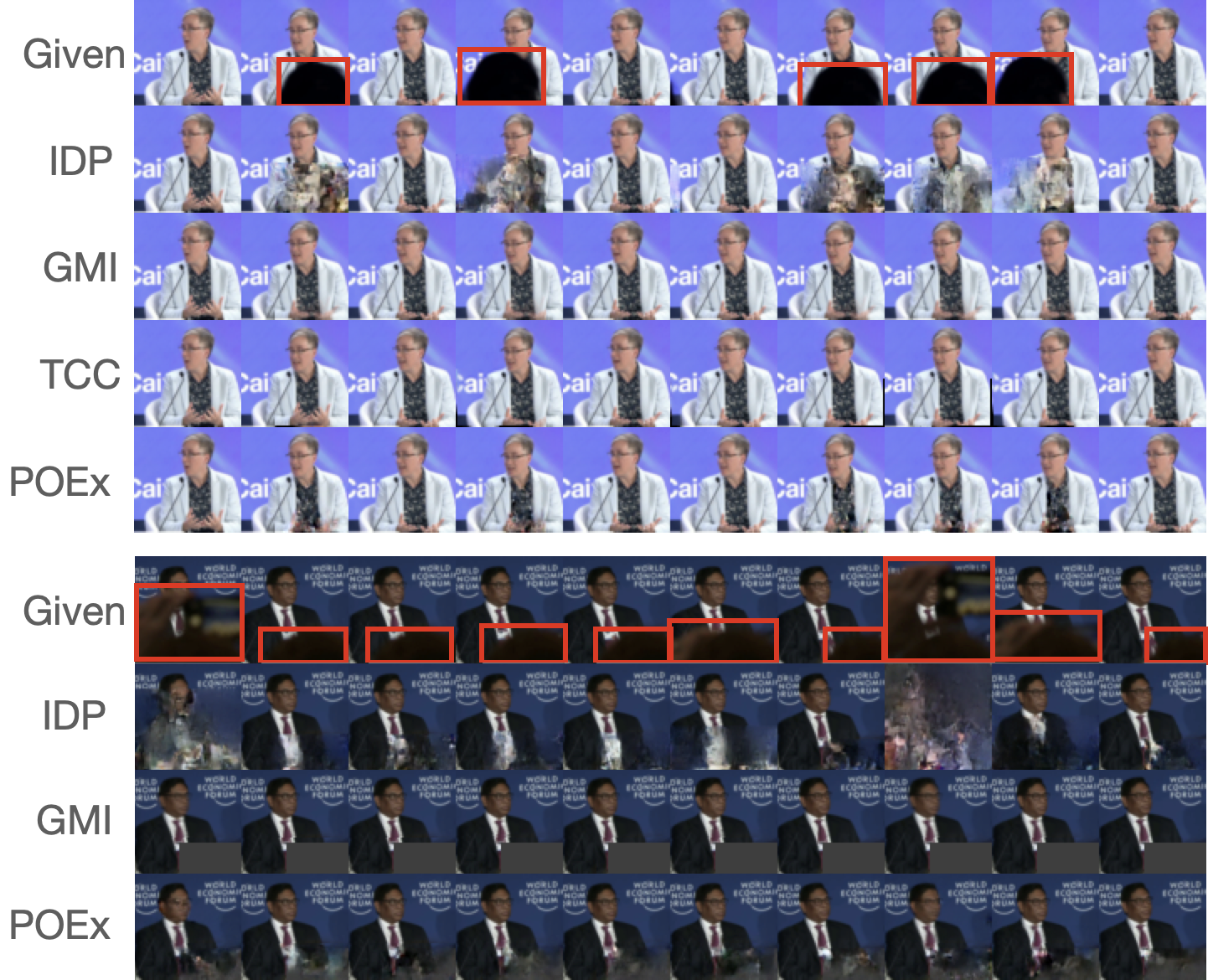}
    \vspace{-10pt}
    \caption{Additional examples for occlusion removal.}
    \label{fig:video_supp}
\end{figure}

\begin{figure}
    \centering
    \includegraphics[width=\linewidth]{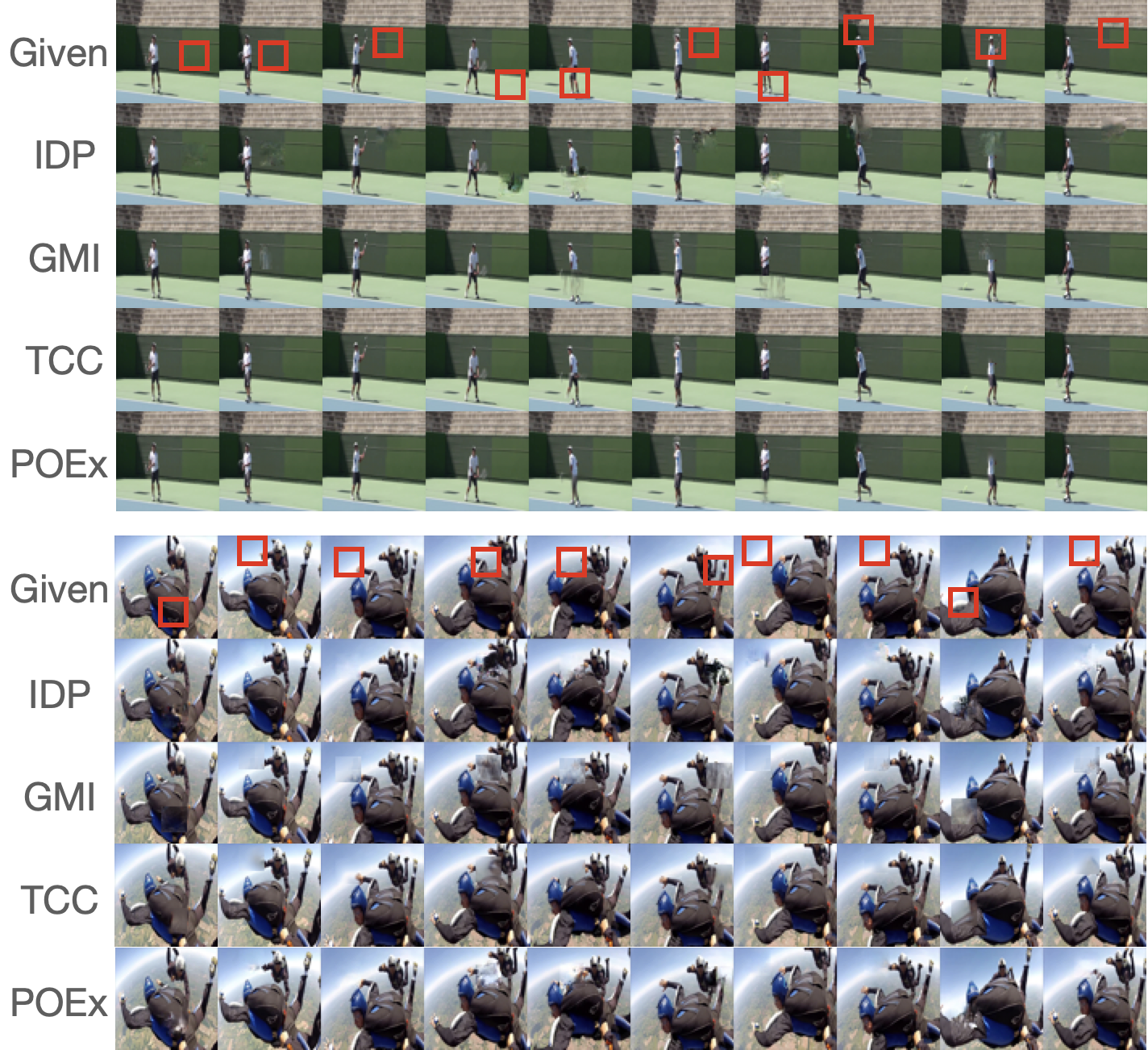}
    \vspace{-10pt}
    \caption{Additional examples for inpainting Youtube videos.}
    \label{fig:youtube_supp}
\end{figure}

\subsection{Set of Functions}
Generalizing the concept of partial sets to infinite dimensional set elements, we can utilize \method to model a set of functions. Each function is represented as a set of (input, target) pairs. We first evaluate on the multi-task Gaussian processes. Following \cite{bonilla2008multi}, we directly specify the correlations among functions: 
\begin{equation*}
    \langle f_l(x), f_k(x') \rangle = K_{lk}^{f}k^{x}(x, x'), \quad y_{il} = \mathcal{N}(f_l(x_i), \sigma_l^2),
\end{equation*}
where $K^f$ is a positive semi-definitive matrix that specifies the inter-task similarities, $k^x$ is a covariance function over inputs. To model $N$ functions, $K^f$ is a $N \times N$ matrix and $K^f_{lk}$ is the element in row $l$ and column $k$. Here, we assume the tasks are permutation equivariant, that is, every two tasks have the same correlation:
\begin{equation*}
    K^f_{lk} = 
\left\{
\begin{array}{lr}
    c,  l \neq k \\
    1,  l = k
\end{array}
\right.
\end{equation*}
Similar to Neural processes, we use a Gaussian kernel for $k^x$. During training, we generate synthetic data from 5 functions with $c=0.9$. We sample at most 100 points from each function and select at most 10 points as context. The sampled points are then transformed by shifting or reversing to obtain 5 target functions. Our \method model is trained by conditioning on a one-hot function identifier.

Following Neural Processes, we also build a set of functions from a set of images. Given a set of MNIST images from the same class, we view each image as a function between the pixel index and its value. During training, we sample an arbitrary subset from each image as context.

\end{document}